\documentclass{alt2019}

\usepackage{times}

\usepackage{amsmath}
\usepackage{hyperref}
\usepackage{bm}
\usepackage{wrapfig}
\usepackage{xcolor,colortbl}
\usepackage{enumitem}
\usepackage{thmtools,thm-restate}
\usepackage[export]{adjustbox}
\usepackage{verbatim}

\definecolor{blued}{RGB}{70,197,221}
\definecolor{pearOne}{HTML}{2C3E50}
\definecolor{pearTwo}{HTML}{A9CF54}
\definecolor{pearTwoT}{HTML}{C2895B}
\definecolor{pearThree}{HTML}{E74C3C}
\colorlet{titleTh}{pearOne}
\colorlet{bull}{pearTwo}
\definecolor{pearcomp}{HTML}{B97E29}
\definecolor{pearFour}{HTML}{588F27}
\definecolor{pearFith}{HTML}{ECF0F1}
\definecolor{pearDark}{HTML}{2980B9}
\definecolor{pearDarker}{HTML}{1D2DEC}

\usepackage[colorinlistoftodos, textwidth=20mm]{todonotes}

\let\originalleft\left
\let\originalright\right
\renewcommand{\left}{\mathopen{}\mathclose\bgroup\originalleft}
\renewcommand{\right}{\aftergroup\egroup\originalright}

\definecolor{graphicbackground}{rgb}{0.96,0.96,0.8}
\definecolor{rouge1}{RGB}{226,0,38}  
\definecolor{orange1}{RGB}{243,154,38}  
\definecolor{jaune}{RGB}{254,205,27}  
\definecolor{blanc}{RGB}{255,255,255} 
\definecolor{rouge2}{RGB}{230,68,57}  
\definecolor{orange2}{RGB}{236,117,40}  
\definecolor{taupe}{RGB}{134,113,127} 
\definecolor{gris}{RGB}{91,94,111} 
\definecolor{bleu1}{RGB}{38,109,131} 
\definecolor{bleu2}{RGB}{28,50,114} 
\definecolor{vert1}{RGB}{133,146,66} 
\definecolor{vert3}{RGB}{20,200,66} 
\definecolor{vert2}{RGB}{157,193,7} 
\definecolor{darkyellow}{RGB}{233,165,0}  
\definecolor{lightgray}{rgb}{0.9,0.9,0.9}
\definecolor{darkgray}{rgb}{0.6,0.6,0.6}
\definecolor{babyblue}{rgb}{0.54, 0.81, 0.94}
\definecolor{citrine}{rgb}{0.89, 0.82, 0.04}
\definecolor{misogreen}{rgb}{0.25,0.6,0.0}

\DeclareMathOperator*{\argmax}{arg\,max}

\renewcommand{\d}[1]{\ensuremath{\operatorname{d}\!{#1}}}

\newcommand{\I}{{\mathds{1}}}

\newcommand{\R}{\mathbb{R}}

\newcommand{\E}{\mathbb{E}}

\newcommand{\pa}[1]{\left(#1\right)}

\newcommand{\card}[1]{\left|#1\right|}

\newcommand{\CommaBin}{\mathbin{\raisebox{0.5ex}{,}}}

\newcommand{\cO}{\mathcal{O}}
\newcommand{\tcO}{\widetilde{\cO}}

\newcommand{\cX}{\mathcal{X}}
\newcommand{\X}{\cX}

\renewcommand{\epsilon}{\varepsilon}
\renewcommand{\hat}{\widehat}
\renewcommand{\tilde}{\widetilde}
\renewcommand{\bar}{\overline}

\newcommand{\nothere}[1]{}

\usepackage{xspace}

\newcommand{\StroquOOL}{\normalfont \texttt{\textcolor[rgb]{0.5,0.2,0}{StroquOOL}}\xspace}
\newcommand{\SequOOL}{\normalfont{\texttt{\textcolor[rgb]{0.5,0.2,0}{SequOOL}}}\xspace}

\newcommand{\StoSOO}{\texttt{StoSOO}\xspace}
\newcommand{\POO}{\texttt{POO}\xspace}
\newcommand{\DOO}{\texttt{DOO}\xspace}
\newcommand{\SOO}{\texttt{SOO}\xspace}
\newcommand{\Zooming}{\texttt{Zooming}\xspace}

\newcommand{\HCT}{\texttt{HCT}\xspace}

\newcommand{\HOO}{\texttt{HOO}\xspace}

\newcommand{\Direct}{\texttt{DiRect}\xspace}

\let\inf\undefined
\DeclareMathOperator*{\inf}{\vphantom{\sup}inf}

\DeclareBoldMathCommand{\I}{I}
\DeclareBoldMathCommand{\e}{e}
\DeclareBoldMathCommand{\f}{f}
\DeclareBoldMathCommand{\g}{g}
\DeclareBoldMathCommand{\a}{a}
\DeclareBoldMathCommand{\b}{b}
\DeclareBoldMathCommand{\d}{d}
\DeclareBoldMathCommand{\m}{m}
\DeclareBoldMathCommand{\p}{p}
\DeclareBoldMathCommand{\q}{q}
\DeclareBoldMathCommand{\v}{v}
\DeclareBoldMathCommand{\V}{V}
\DeclareBoldMathCommand{\x}{x}
\DeclareBoldMathCommand{\t}{t}
\DeclareBoldMathCommand{\X}{X}
\DeclareBoldMathCommand{\Y}{Y}
\DeclareBoldMathCommand{\z}{z}
\DeclareBoldMathCommand{\Z}{Z}
\DeclareBoldMathCommand{\M}{M}
\DeclareBoldMathCommand{\n}{n}
\DeclareBoldMathCommand{\ssigma}{\sigma}
\DeclareBoldMathCommand{\SSigma}{\Sigma}
\DeclareBoldMathCommand{\OOmega}{\Omega}
\DeclareBoldMathCommand{\y}{y}
\DeclareBoldMathCommand{\U}{U}
\DeclareBoldMathCommand{\w}{w}
\DeclareBoldMathCommand{\W}{W}
\DeclareBoldMathCommand{\L}{L}

\DeclareBoldMathCommand{\s}{s}
\DeclareBoldMathCommand{\S}{S}
\DeclareBoldMathCommand{\A}{A}
\DeclareBoldMathCommand{\B}{B}
\DeclareBoldMathCommand{\C}{C}
\DeclareBoldMathCommand{\D}{D}
\DeclareBoldMathCommand{\E}{\mathbb{E}}
\DeclareBoldMathCommand{\G}{G}
\DeclareBoldMathCommand{\H}{H}
\DeclareBoldMathCommand{\P}{\mathbb{P}}
\DeclareBoldMathCommand{\Q}{Q}
\DeclareBoldMathCommand{\R}{R}
\DeclareBoldMathCommand{\X}{X}
\DeclareBoldMathCommand{\mmu}{\mu}
\DeclareBoldMathCommand{\ones}{1}
\DeclareBoldMathCommand{\zeros}{0}

\newcommand{\tree}{\mathcal{T}}

\newcommand{\maxNumberChildren}{K}
\newcommand{\depthOp}{\bot}
\newcommand{\partition}{\mathcal{P}}

\newcommand{\timeHorizon}{n}

\newcommand{\pullsNumber}{T}

\newcommand{\Exp}{\E}

\newcommand{\TODO}[1]{
\ifmmode
\text{\textcolor{red}{TODO: #1}}
\else
\textcolor{red}{TODO: #1}
\fi
}

\renewcommand{\epsilon}{\varepsilon}
\renewcommand{\hat}{\widehat}
\renewcommand{\tilde}{\widetilde}
\renewcommand{\bar}{\overline}

\newcommand{\Real}{\mathbb{R}}
\newcommand{\Integer}{\mathbb{N}}

\newcommand{\dom}{\mathcal X}

\newcommand{\hmax}{h_{\rm max}}
\newcommand{\pmax}{p_{\rm max}}
\newcommand{\lambertW}{W}

\newtheorem{assumption}{Assumption}

\newcounter{scratchcounter}

\firstpageno{1}

\usepackage{hyperref}
\hypersetup{
	colorlinks,
	citecolor=pearDark,
	linkcolor=pearThree,
	breaklinks=true,
	urlcolor=pearDarker}

\begin{document}

	\title[simple approach to optimization under a minimal smoothness assumption]{A simple parameter-free and adaptive approach to optimization under a minimal local smoothness assumption}

\coltauthor{\Name{Peter L.\,Bartlett} \Email{peter@berkeley.edu}\\
	\addr University of California, Berkeley, USA
	\AND
	\Name{Victor Gabillon} \Email{victor.gabillon@huawei.com}\\
	\addr Noah's Ark Lab, Huawei Technologies, London, UK
		\AND
	\Name{Michal Valko} \Email{michal.valko@inria.fr}\\
	\addr SequeL team, INRIA Lille - Nord Europe, France
		\vspace{-1em}
}
	
	\maketitle
\begin{abstract}
	We study the problem of optimizing a function under a \emph{budgeted number of evaluations}. 
	We only assume that the function is \emph{locally} smooth around one of its
	global optima.
	The difficulty of optimization is measured in terms of 1) the amount of \emph{noise} $b$ of the function evaluation and 2)
	the local smoothness, $d$, of the function. A smaller $d$ results in smaller optimization error.
	We come with a new, simple, and parameter-free approach. First, for all values of $b$ and $d$, this approach recovers at least the state-of-the-art regret guarantees. Second, our approach additionally obtains these results while being \textit{agnostic} to the values of both $b$ and $d$. This leads to the first algorithm that naturally adapts to an \textit{unknown} range of noise $b$ and leads to significant improvements in a moderate and low-noise regime.
	Third, our approach also obtains a remarkable improvement over the state-of-the-art \SOO algorithm when the noise is very low which includes the case of optimization under deterministic feedback ($b=0$).
	There, under our minimal local smoothness assumption, this improvement is of exponential magnitude and holds for a class of functions that covers the vast majority of functions that practitioners optimize ($d=0$).
	We  show that our algorithmic improvement is borne out in  experiments as we empirically show faster convergence on common benchmarks.
\end{abstract}

	\begin{keywords}
		optimization, tree search, deterministic feedback, stochastic feedback
	\end{keywords}
	
	%
	\section{Introduction}
	In budgeted function optimization, a learner optimizes a function $f: \dom \rightarrow\ \Real$ having access to a number of evaluations limited by $\timeHorizon$. For each of the $\timeHorizon$ evaluations (or rounds), at round~$t$,  the learner picks an element $x
	_t\in\dom$  and observes  a real number $y_t$, where $y_t=f(x_t)+\epsilon_t$, where $\epsilon_t$ is the noise. Based on $\epsilon_t$,  we distinguish two feedback cases:
	\begin{description}
	\item[Deterministic feedback] 	The   evaluations are  noiseless, that is $\forall t$, $\epsilon_t=0$ and $y_t=f(x_t)$.
	Please refer to the work by \cite{defreitas2012exponential}  for a  motivation, many applications, and references on the importance of the case $b=0$.
	\item[Stochastic feedback] The  evaluations  are
	perturbed by a noise of range $b\in\Real_+$\footnote{Alternatively, we can turn the boundedness assumption into a sub-Gaussianity assumption equipped with a variance parameter equivalent to our range $b$.}: 
	At any round, $\epsilon_t$ is a random variable,  assumed
	independent of the noise at previous rounds,
	\begin{equation} \label{eq:store}
	\Exp \left[ y_t|x_t\right] = f(x_t)\quad \text{ and } \quad |y_t-f(x_t)|\leq b.
	\end{equation}
	\end{description}
		The objective of the learner is to return an element $x(\timeHorizon)\in\dom$ with largest possible value $f\left(x(\timeHorizon)\right)$ after the $\timeHorizon$ evaluations. 
	$x(\timeHorizon)$ can be different from the last evaluated element~$x_n$. 
	More precisely,
	the performance
	of the algorithm is the loss (or simple regret), 
	\[
	r_\timeHorizon \triangleq \sup_{x\in\mathcal X} f\left(x\right)- f\left(x(\timeHorizon)\right)\,.
	\]
	We consider the case that the evaluation is costly. Therefore, we minimize~ $r_\timeHorizon$ as a function of $\timeHorizon$. We assume that there exists at least one point
	$x^\star\in \mathcal X$  such that
	$f(x^\star) = \sup_{x\in\mathcal X} f(x)$.
	%

	\paragraph{Prior work} Among the large work on optimization, we focus on algorithms that perform well under \emph{minimal} assumptions as well as minimal knowledge about the function. Relying on minimal assumptions means that we target functions that are particularly hard to optimize. For instance, we may not have access to the gradients of the function, gradients might not  be well defined, or the function may not be continuous.
	While some prior works assume a \emph{global} smoothness of the function~\citep{Pinter13GO,strongin2000global,Hansen2003GO,Kearfott2013RG}, another line of research assumes only a \emph{weak}/\emph{local} smoothness around one global maximum~\citep{Auer07IR,Kleinberg08MA,bubeck2011x}.  
	However, within this latter group, some algorithms require the knowledge of the local smoothness such as \HOO~\citep{bubeck2011x}, \Zooming~\citep{Kleinberg08MA}, or \DOO~\citep{Munos11OO}. Among the works relying on an \emph{unknown} local  smoothness,
	\SOO~\citep{Munos11OO,kawaguchi2016global} 
	represents the state-of-the-art for the deterministic feedback.
	For the stochastic feedback, \StoSOO~\citep{Valko88SS} extends \SOO for a limited class of functions. \POO~\citep{Grill15BB} provides more general results. 
	We classify the most related algorithms in the following table.
	\begin{table}[h!]
		\vspace{-.1cm}
		\center
		\begin{tabular}{|>{\columncolor[gray]{.95}}lcc|}
			\hline
			\textbf{smoothness} & \textbf{deterministic} & \textbf{stochastic}\\
			\hline
			known& \DOO & \Zooming, \HOO\\
			unknown &\Direct, \SOO, \SequOOL & \StoSOO, \POO, \StroquOOL\\
			\hline
		\end{tabular}
		\label{tabloo}\vspace{-.2cm}
	\end{table} 
	
	\noindent Note that for more specific assumptions on the smoothness, some works study optimization without the knowledge of smoothness:  \Direct~\citep{jones1993lipschitzian} and others~\citep{slivkins2011multi-armed,bubeck2011lipschitz,Malherbe17GO} tackle  Lipschitz optimization.
	
	Finally, there are algorithms that instead of simple regret, optimize
	\emph{cumulative regret}, like \HOO~\citep{bubeck2011x} or \HCT \citep{azar2014online}.
	Yet,  none of them adapts to the unknown smoothness and 
	compared to them, the algorithms for simple regret that are able to 
	do that, such as \POO or  our \StroquOOL, need to explore significantly more, which negatively impacts  their cumulative regret~\citep{Grill15BB,locatelli2018adaptivity}.
	
	\paragraph{Existing tools}\textcolor{gray}{\textbf{Partitionining and near-optimality dimension\ }}  
	As in most of the previously mentioned work, the search domain $\dom$ is partitioned into cells at different scales (depths), i.e., at a deeper depth, the cells are smaller but still cover all of $\dom$. The objective of many algorithms is  to explore the value of~$f$ in the cells of the partition and determine at the deepest \textit{depth} possible in which cell is \emph{a} global maximum of the function. The notion of near-optimality dimension~$d$ aims at capturing the smoothness of the function and characterizes the complexity of the optimization task. 
	We adopt the definition of near-optimality dimension given recently by~\citet{Grill15BB} that unlike \cite{bubeck2011x}, \cite{Valko88SS}, \cite{Munos11OO}, and \cite{azar2014online},
	avoids topological notions and does not artificially attempt to separate the difficulty of the optimization from the partitioning. For each depth $h$, it simply counts the number of near-optimal cells $\mathcal{N}_h$, cells whose value is close to~$f(x^\star)$, and determines how this number evolves with the depth $h$.
	The smaller~$d$, the more accurate the optimization should be.
	\paragraph{New challenges} \textcolor{gray}{\textbf{Adaptations to different data complexities\ }} 
	As did~\cite{Bubeck12BB}, \cite{Seldin14OP}, and~\cite{DeRooij2014FT} in other contexts, we design algorithms that demonstrate near-optimal behavior under data-generating processes of different nature, obtaining the \textit{best of all these possible worlds}. 
	In this paper, we consider the two following data complexities for which we bring new improved adaptation.\vspace{-.1cm}
	\begin{itemize}[leftmargin=.5cm]
		\item \textit{near-optimality dimension $d=0$}: In this case, the number of near-optimal cells is simply bounded by a constant that does not depend on $h$. As shown by~\cite{Valko88SS}, if  the function is lower-  and upper-bounded by two polynomial envelopes of the same order  around a global optimum, then $d=0$. 
		As discussed in the book of~\citet[section 4.2.2]{munos2014from}, $d=0$ 
		covers the vast majority of functions that practitioners optimize and the functions with $d>0$ given as examples 
		in prior work \citep{bubeck2011lipschitz,Grill15BB,Valko88SS,Munos11OO,shang2019general} are carefully engineered.
		Therefore, the case of $d=0$ is of practical importance. However, even with deterministic feedback, the case $d=0$ with unknown smoothness has not been known to have a learner with a near-optimal guarantee. In this paper, we also provide that. Our approach not only adapts very well to the case $d=0$ and $b\approx 0$, it also provides an \emph{exponential} improvement over the state of the art for the simple regret rate.  \vspace{-.05cm}
		\item \textit{low or moderate noise regime}: When facing a noisy feedback, most algorithms assume that the noise is of a \emph{known} predefined range, often using $b=1$ hard-coded in their use of upper confidence bounds.  Therefore, they \emph{cannot} take advantage of low noise scenarios. 
		Our algorithms have a regret that scales with the range of the noise~$b$, \emph{without a prior knowledge of~$b$.} Furthermore, our algorithms ultimately recover the new improved rate of the deterministic feedback suggested in the precedent case ($d=0$).
	\end{itemize}\vspace{-.1cm}
	\paragraph{Main results} \textcolor{gray}{\textbf{Theoretical results and empirical performance\ }}
	We consider the optimization under an unknown \textit{local} smoothness. We design two algorithms, \SequOOL for the deterministic case in Section~\ref{deterministic} and \StroquOOL for the stochastic one in Section~\ref{stochastic}. 
	\begin{itemize}[leftmargin=.5cm]
		\item \SequOOL is the first algorithm to obtain a loss $e^{-\tilde\Omega(n)}$ under such minimal assumption, with deterministic feedback. The previously known \SOO~\citep{Munos11OO} is only proved to achieve a loss of $\cO(e^{-\sqrt{n}})$. 
		Therefore, \SequOOL achieves, up to log factors, the result of \DOO that  \emph{knows the smoothness}. 
		Note that \cite{kawaguchi2016global} designed a new version of \SOO, called \texttt{LOGO}, that gives more flexibility in exploring more local scales but it was still only shown to achieve a loss of $\cO(e^{-\sqrt{n}})$ despite the introduction of a new parameter.
		Achieving exponentially decreasing regret had previously only been achieved in setting with more assumptions~\citep{defreitas2012exponential,Malherbe17GO,Kawaguchi15BO}. 
		For example, \cite{defreitas2012exponential} achieves $e^{-\tilde\Omega(n)}$ regret assuming several assumptions, for example that 
		the function $f$ is sampled from the Gaussian process with four times differentiable kernel along the diagonal. The consequence of our results is that to achieve $e^{-\tilde\Omega(n)}$ rate, none of these strong assumptions is necessary.
		\item \StroquOOL  recovers, in the stochastic feedback, up to log factors, the results of \POO, 
		for the same assumption. However, as discussed later, \StroquOOL is a simpler approach than \POO which additionally features much simpler and elegant analysis.	
		\item  \StroquOOL  adapts naturally to different noise range, i.e., the various values of $b$.
		\item \StroquOOL obtains the \textit{best of both worlds} in the sense that \StroquOOL also obtains, up to log factors, the new optimal rates reached by \SequOOL in the deterministic case. \StroquOOL obtains this result without being aware a priori of the nature of the data, only for an additional $\log$ factor.
		Therefore, if we neglect the additional log factor, we can just have a single algorithm, \StroquOOL, that performs well in both deterministic and stochastic case, without the knowledge of the smoothness in either one of them.
		\item In the numerical experiments,   \StroquOOL naturally adapts to lower noise. \SequOOL obtains an exponential regret decay when $d=0$ on common benchmark functions.
	\end{itemize}
	%
	%
	\paragraph{Algorithmic contributions and originality of the proofs} \textcolor{gray}{\textbf{Why does it work?}}
	Both \SequOOL and \StroquOOL 	are simple and parameter-free algorithms. 
	Moreover, both \SequOOL and \StroquOOL are based on a new core idea that the search for the optimum should progress strictly \textit{sequentially}  from an exploration of shallow depths (with large cells) to deeper depths (small and localized cells). This is different from the standard approach in \SOO, \StoSOO, and the numerous extensions that  \SOO  has inspired~\citep{Busoniu13OP,wang2014bayesian,Al-Dujaili18MO,Qian16SS,Kasim2016ID,Derbel15SO,Preux14BA,bucsoniu2014consensus,kawaguchi2016global}.
We come up with our idea by identifying a bottleneck in \SOO~\citep{Munos11OO} and its extensions that open all depths \textit{simultaneously} (their Lemma $2$). However, in general, we show that the improved exploration of the shallow depths is beneficial for the deeper depths and therefore, we always complete the exploration of depth $h$ before going to depth $h+1$. As a result, we design a more \textit{sequential} approach that simplifies our Lemma~$2$.
	
	This desired simplicity is also achieved by  being the first to adequately leverage  the reduced and natural set of assumptions introduced in the \POO paper~\citep{Grill15BB}. 
	This adequate and simple leverage should not conceal the fact that our local smoothness assumption is minimal and already way weaker  than global  Lipschitzness.    
	Second, this leveraging was absent in the analysis for \POO which additionally relies on the 40 pages proof of \HOO; see \citealp{shang2019general} for a detailed discussion. Our proofs are succinct\footnote{The proof is even redundantly written twice for \StroquOOL and \SequOOL for completeness} while obtaining performance improvement ($d=0$) and a new adaptation ($b=0$). 
	To obtain these, in an original way, our theorems are now based on solving a transcendental equation with the  \textbf{Lambert $W$ function}.
	For \StroquOOL, a careful discrimination of the parameters of the equation leads to  optimal rates both in the deterministic and stochastic case.
	
	Intriguingly, the amount of evaluations allocated to each depth $h$ follows a \textbf{Zipf law}~\citep{Powers98AE}, that is,
	each depth level $h$ is simply  pulled inversely proportional to its depth index $h$. 
	It provides a parameter-free method to explore the depths without knowing the bound $C$ on the number of optimal cells per depth ($\mathcal{N}_h = C \propto \timeHorizon/h$ when $d=0$) and obtain a maximal optimal depth $h^\star$ of order $n/C$.
	A Zipf law has been used by~\cite{Audibert10BA} and~\cite{abbasi2018best} in pure-exploration bandit problems but without any notion of depth in the search.
	In this paper, we introduce the Zipf law to tree search.
	
	Finally, another novelty is that were are \emph{not using upper bounds} in \StroquOOL (unlike \StoSOO, \HCT, \HOO, \POO), which results in  the contribution of \textit{removing the need to know the noise amplitude.} 
	\section{Partition, tree, assumption, and near-optimality dimension}
	\label{sectionsorigourous}\vspace{-0.05cm}
	\paragraph{Partitioning}
	The 
	hierarchical partitioning
	$\partition
	=
	\{\partition_{h,i}\}_{h,i}$
	we consider is similar to the ones introduced in prior work~\citep{Munos11OO,Valko88SS,Grill15BB}: For any
	depth
	$h\geq
	0$
	in the tree representation, the set
	$\{\partition_{h,i}\}_{1\leq i \leq I_h}$
	of
	\textit{cells}  (or nodes)
	forms a partition of
	$\dom$, where
	$I_h$
	is the number of cells at depth
	$h$. At depth $0$, the root of the tree, there is a single cell
	$\partition_{0,1}=\dom$. A cell $\partition_{h,i}$ of depth $h$ is split into children subcells $\{\partition_{h+1,j}\}_j$
	of depth $h+1$. 
	As \cite{Grill15BB}, our  work  defines   a  notion  of  near-optimality  dimension $d$
	that  does  not  directly  relate
	the  smoothness  property  of
	$f$
	to  a  specific  metric
	$\ell$
	but \emph{directly}
	to  the
	hierarchical  partitioning
	$\partition
	$. Indeed, an interesting
	fundamental quest is to determine a good characterization of the difficulty of the optimization
	for an algorithm that uses a given hierarchical partitioning of the space $\dom$
	as its input~\citep[see][for a detailed discussion]{Grill15BB}.
	Given a global maximum $x^\star$ of $f$,
	$i^\star_{h}$
	denotes the index of the unique cell of depth
	$h$
	containing
	$x^\star$
	, i.e., such that
	$x^\star\in \partition_{h,i^\star_{h}}$.
	We follow the work of~\cite{Grill15BB} and state a \emph{single} assumption on
	both the partitioning
	$\partition$
	and the function~$f$.
	\begin{assumption}\label{as:smooth}
		For any global optimum $x^\star$, there exists $\nu>0$ and $\rho\in (0,1)$ such that $\forall h \in \Integer $,
		$\forall x \in \mathcal P_{h,i^\star_h},  f(x)    \geq f(x^\star) - \nu\rho^h.$
	\end{assumption}
	%
	\begin{definition}\label{def:neardim}
		For any $\nu > 0$, $C>1$, and $\rho \in (0,1)$, the \textbf{near-optimality dimension}\footnote{\cite{Grill15BB} define  $d(\nu,C,\rho)$ with the constant 2 instead of 3.  3 eases the exposition of our results.} $d(\nu,C,\rho)$ of~$f$ with respect to the partitioning
		$\mathcal P$ and with associated constant $C$,
		is 
		%
		\[
		d(\nu,C,\rho) \triangleq    \inf \left\{d'\in\Real^+:  \forall h \geq 0, ~\mathcal N_h(3\nu\rho^h) \leq     C\rho^{-d'h} \right\}\!\CommaBin
		\]
		where
		$\mathcal N_h(\epsilon)$ is the number of cells
		$\mathcal P_{h,i}$    of depth $h$ such that
		$\sup_{x\in \mathcal P_{h,i}} f(x) \geq    f(x^\star) - \epsilon$.
	\end{definition}
	%
	%
	%
	\paragraph{Tree-based learner}
	Tree-based exploration or tree search algorithm is an approach that has been widely applied to optimization as well as bandits or planning \citep{kocsis2006bandit,coquelin2007bandit,hren2008optimistic}; see \cite{munos2014from} for a survey.
	%
	At each round, the learner selects a cell $\partition_{h,i}$ containing a predefined representative element $x_{h,i}$ and asks for its evaluation.
	We denote its value as $f_{h,i}\triangleq f(x_{h,i})$. We use
	$\pullsNumber_{h,i}$ to denote the total number of evaluations  allocated by the learner to the cell $\partition_{h,i}$.
	Our learners  collect the evaluations of~$f$ and organize them in a tree structure $\tree$ that is simply a subset of  $\partition$: $\tree \triangleq \{\partition_{h,i}\in\partition: \pullsNumber_{h,i}>0 \}$, $\tree\subset\partition$.
	%
        For the noisy case, we also define the estimated value of the cell 
	$\hat f_{h,i}$. Given the $\pullsNumber_{h,j}$ evaluations $y_1,\ldots,y_{\pullsNumber_{h,j}},$ we have $ \hat f_{h,i} \triangleq \frac{1}{\pullsNumber_{h,j}} \sum_{s=1}^{\pullsNumber_{h,j}} y_s$,
	the  empirical  average  of  rewards  obtained  at this cell.
	We say that the learner \textit{opens} a cell $\partition_{h,i}$ with $m$ evaluations if it asks for $m$ evaluations from each of the children cells of cell $\partition_{h,i}$.
	In the  deterministic feedback,  $m=1$.  
	For the sake of simplicity, the bounds reported in this paper are in terms of the total number of openings~$\timeHorizon$,  instead of evaluations. The  number of
	function evaluations is  upper bounded by $\maxNumberChildren\timeHorizon$, where
	$\maxNumberChildren$ is the maximum number of children cells of any cell in $\partition$. 
	%

	Our results use the \textbf{Lambert~$W$ function}. 
	Solving for the variable $z$, the equation $A = ze^z$ gives $z=W(A)$. Notice that $W$ is multivalued for $z\leq 0$. Nonetheless, in this paper, we consider $z\geq 0$ and $W(z)\geq 0$, referred to as  the \textit{standard} $W.$
	Lambert $W$  cannot be expressed with elementary functions.
	Yet, due to \citet{Hoorfar08IO}, we have  $W(z)= \log\left(z/\log z\right) + o(1)$. 
	
	 Finally,
	let $[a:c]=\{a,a+1,\ldots,c\}$ with
	$a,c\in\Integer$, $a\leq c$, and $[a]=[1:a]$. Next,
	$\log_{d}$ denotes the logarithm in base $d$, $d\in\Real$. Without a subscript, $\log$ is the natural logarithm in base $e$.
	\section{Adaptive deterministic optimization and improved rate} 
	\label{deterministic}
	\subsection{The {\SequOOL} algorithm}
	\begin{wrapfigure}{r}{.55\textwidth}
		\vspace{-0.4cm}	\centering
		\framebox{
			~    \begin{minipage}{.5\textwidth}
				
				\textbf{Parameters:} 
				$\timeHorizon$, $\partition
				=
				\{\partition_{h,i}\}$\\[.05cm] 
				\textbf{Initialization:}
				Open $\partition_{0,1}$.             $\hmax \gets  \left\lfloor\timeHorizon/\bar\log(\timeHorizon)\right\rfloor\cdot$ \\[.1cm]
				\textbf{For} $h=1$ to $\hmax$
				
				\vspace{.05cm}
				
				$\quad$ Open 
				$\left\lfloor \hmax/h\right\rfloor $ cells $\partition_{h,i}$ of depth $h$ \\\vspace{.05cm}$\qquad$ with largest values  $ f_{h,j}$.
				
				\textbf{Output} $ x(\timeHorizon) \gets \argmax\limits_{x_{h,i}: \partition_{h,i}\in \tree}f_{h,i}$.
			\end{minipage}    
		}
		\caption{The \SequOOL algorithm}\label{fig:hotcold}\vspace{-.25cm}
	\end{wrapfigure}
	The \underline{Sequ}ential Online Optimization aLgorithm \SequOOL is described in Figure~\ref{fig:hotcold}.
	\SequOOL explores the depth sequentially,   one by one, going deeper and deeper with a decreasing number of cells opened per depth~$h$, $\left\lfloor \hmax/h\right\rfloor $ openings at depth~$h$.  The maximal depth that is opened is $\hmax$.
	The analysis of \SequOOL shows that it is useful that $\hmax \triangleq \left\lfloor \timeHorizon/\bar\log\, \timeHorizon\right\rfloor$, where $\bar\log\, \timeHorizon$ is the $\timeHorizon$-th harmonic number,
	$\bar\log\,\timeHorizon\triangleq\sum_{t=1}^\timeHorizon \frac{1}{t}$ with $\bar\log\,\timeHorizon\leq \log\timeHorizon+1$ for any positive integer~$\timeHorizon$.
	\SequOOL returns the element of the evaluated cell with the
	highest value, $x(\timeHorizon)=\argmax\limits_{x_{h,i}: \partition_{h,i}\in \tree}f_{h,i}$.
	%
	%
	%
	%
	%
	%
	%
	%
	%
	%
	%
	%
	%
	%
	We use the  budget of $n+1$ for the simplicity of stating our guarantees.  Notice that \SequOOL does not use more openings than that as
	\begin{align*}
	1+\sum_{h=1}^{\hmax}\left\lfloor \frac{\hmax}{h}\right\rfloor 
	\leq 1+\hmax\sum_{h=1}^{\hmax}  \frac{1}{h}
	= 1+\hmax  \bar\log\,\hmax
	\leq \timeHorizon+1.
	\end{align*}
%
%
	\begin{remark}
		The algorithm can be made anytime and unaware of $\timeHorizon$ using the classic `doubling trick'.
	\end{remark}
	\begin{remark}[More efficient use of the budget]\label{rem:effibud}
		Because of the use of the floor functions $\left\lfloor \cdot\right\rfloor$, the budget used in practice,
		$1+\sum_{h=1}^{\hmax}\left\lfloor \frac{\hmax}{h}\right\rfloor 
		$, can be significantly smaller than $\timeHorizon$. While this only affects numerical constants in the bounds, in practice, it can influence the performance noticeably. Therefore one should consider, for instance, having $\hmax$ replaced by $c\times\hmax$ with $c\in\Real$ and $c=\max \{c'\in\Real: 1+\sum_{h=1}^{\hmax}\left\lfloor \frac{c'\hmax}{h}\right\rfloor \leq \timeHorizon \}$. Additionally,
		the use the budget $\timeHorizon$ could be slightly optimized by taking into account that the necessary number of pulls at depth $h$ is actually $\min\pa{\left\lfloor \hmax/h\right\rfloor, K^h}$.
	\end{remark}
	%
	%
	%
	%
	\subsection{Analysis of \SequOOL}
	%
	%
	For any global optimum $x^\star$ in $f$, let 
	$\depthOp_h$ be
	the depth of the deepest opened node containing~$x^\star$ 
	at the end of the opening of depth 
	$h$ by \SequOOL---an iteration of the \textbf{for} cycle. Note that~$\depthOp_{(\cdot)}$ is increasing. The proofs of the following statements are given in Appendix~\ref{app:firstone}.
	\begin{restatable}{lemma}{restalemhstar}\label{lem:hstar}
		For any global optimum $x^\star$ with associated $(\nu,\rho)$ as defined in Assumption~\ref{as:smooth}, for $C>1$, for any depth that $h \in [\hmax]$, if $ \hmax/h
		\geq  C\rho^{-d(\nu,C,\rho)h}$, we
		have
		$\depthOp_h= h$ with $\depthOp_0=0$.
	\end{restatable}
	%
	\noindent
	Lemma~\ref{lem:hstar} states that as long as  at depth $h$, \SequOOL opens more cells than the number of near-optimal cells at depth $h$, the cell containing $x^\star$ is opened at depth $h$. 
	%
	%
	\begin{restatable}{theorem}{restathsequool}\label{th:sequool}
		Let $\lambertW$ be the standard Lambert $\lambertW$ function (Section~\ref{sectionsorigourous}). 
		For any function~$f$, one of its global optima $x^\star$ with associated $(\nu,\rho)$, $C>1$, and near-optimality dimension $d=d(\nu,C,\rho)$, we  have,
		after $\timeHorizon$ rounds,     the simple regret of \SequOOL is bounded as follows:
		\begin{align*}
		&\text{\textbullet~  If $d=0$, }
		~~ r_\timeHorizon
		\leq
		\nu\rho^{\frac{1}{C} \left\lfloor\frac{\timeHorizon}{\bar\log\,\timeHorizon}\right\rfloor}.~~~ 
		&&\text{\textbullet~  If $d>0$,}        
		~~ r_\timeHorizon
		\leq
		\nu e^{ 
			-\frac{1}{d}\lambertW\left(\frac{d\log(1/\rho)}{C} \left\lfloor\frac{\timeHorizon}{\bar\log\,\timeHorizon}\right\rfloor \right)}.
		\end{align*}
	\end{restatable}
	\noindent
	For more readability, Corollary~\ref{c:readable} uses a lower bound on $W$ by~\citet{Hoorfar08IO}.
	%
	\begin{restatable}{corollary}{restacorohstar}	\label{c:readable}
		If $d>0$,  assumptions in Theorem~\ref{th:sequool} hold and
		$\tilde n \triangleq\left\lfloor\timeHorizon/\bar\log\,\timeHorizon\right\rfloor d\log(1/\rho)/C>e$,
		\[
		r_\timeHorizon
		\leq
		\nu  \left(\frac{\tilde n}{\log\left(\tilde n \right)}\right)^{- \frac{1}{d}}.\]
	\end{restatable}
	\subsection{Discussion for the deterministic feedback}
	\paragraph{Comparison with \SOO}
	\SOO and \SequOOL are both for deterministic optimization without knowledge of the smoothness.
	The regret guarantees of \SequOOL are an improvement over \SOO. While when $d>0$ both algorithms achieve a regret $\tcO\left(\timeHorizon^{-1/d}\right)$, when $d=0$, the regret of \SOO is $\cO(\rho^{\sqrt{\timeHorizon}})$ while the regret of \SequOOL is $\rho^{\tilde\Omega(n)}$ which is a significant improvement. As discussed in the introduction and by \citet[Section~5]{Valko88SS}, the case $d=0$ is very common.
	As pointed out by~\citet[Corollary~2]{Munos11OO}, \SOO has to actually know whether $d=0$ or not
	to set the maximum depth of the tree as a parameter for \SOO. 
	\SequOOL is fully adaptive, does not need to know any of this and actually gets a better rate.%
	\footnote{A similar behavior is also achieved by combining two \SOO algorithms, by running half of the samples for $d=0$ and half for $d>0$.
		However, \SequOOL does this naturally and gets a better rate when $d=0$.}
	%
	
	The conceptual difference from \SOO is that \SequOOL is more sequential: For a given depth $h$, \SequOOL first opens cells at depth $h$ and then at depth $h+1$ and so on, without coming back to lower depths.
	Indeed, an opening at depth $h+1$  is based on the values observed while opening at depth~$h$. Therefore, it is natural and less wasteful to do the openings in a sequential order.
	Moreover, \SequOOL is more conservative as it opens  the lower depths more while \SOO opens every depth equally.
	However from the perspective of \emph{depth},  \SequOOL is more aggressive as it opens depth as high as $\timeHorizon$,
	while \SOO stops at $\sqrt{\timeHorizon}$.
	\paragraph{Comparison with \DOO}
	Contrarily to \SequOOL, \DOO  knows  the smoothness of the function that is used as input parameter $\tilde \nu= \nu$  and $\tilde \rho= \rho$. However this knowledge only improves the logarithmic factor in the current upper bound. When $d>0$, \DOO achieves a simple regret of $\cO\left(\timeHorizon^{-1/d}\right)$, when $d=0$, the simple regret  is of  $\cO\left(\rho^{\timeHorizon}\right)$.\vspace{-.05cm}
	%
	\paragraph{\DOO with multiple parallel $(\tilde \nu, \tilde\rho)$ instances?} 
	An alternative approach to \SequOOL, based on \DOO, which would also 
not require the knowledge of the true smoothness $( \nu, \rho)$,
	 is to run  $m$ multiple parallel instances of \DOO with different values for  $\tilde\nu$ and $\tilde\rho$. For instance, we could mimic the behavior of \POO~\citep{Grill15BB}, and run$m\triangleq\left\lfloor\log n\right\rfloor$  instances of \DOO, each with budget $n/\left\lfloor\log n\right\rfloor$, where, in instance $i\in[\left\lfloor\log n\right\rfloor]$, $\tilde\rho_i$ is set to $1/2^i$.
	Under the condition that $\rho\geq \tilde\rho_{\min}= 1/2^{\left\lfloor\log n\right\rfloor}\approx 1/n$,
	 among these $\left\lfloor\log n\right\rfloor$ instances, one of them, let us say that the $j$-th one, is such that we have $\tilde\rho_j=1/2^j\leq \rho\leq 1/2^{j-1}=2\tilde\rho_j$.	
	 This instance $j$ of \DOO therefore a $x(n)$ with a regret $\rho^{\tilde\Omega(n)}$.

	 However, in the case of
	   $\rho\leq \tilde\rho_{\min}= 1/2^{\left\lfloor\log n\right\rfloor}=1/n$, we  can only guarantee a regret $(\tilde\rho_{\min})^{\tilde\Omega(n)}$.
	Therefore, for a fixed $n$, this approach will fail to capture the case where $\rho\approx 0$ such as, for instance, the case  $\rho= e^{-n}$.
	Note that this argument still holds if the number of parallel instances $m=o(n)$.
Finally, the other disadvantage would be that  as in \POO, this alternative would use upper-bounds $\nu_{\max}$ and $\rho_{\max}$ that would appear in the final guarantees. 	
	%
	\paragraph{Lower bounds} As discussed by~\cite{munos2014from} for $d=0$, \DOO matches the lower bound
and it is even comparable to the lower-bound for concave functions. While \SOO was not matching the bound of \DOO, 
with our result, we now know that, up to a log factor, it is possible to achieve the same performance as \DOO, 
\emph{without the knowledge of the smoothness.}

\section{Noisy optimization with adaptation to low noise}
\label{stochastic}
%
%
%
%
%
\subsection{The {\StroquOOL} algorithm}\label{s:stroalgo}
%
\begin{wrapfigure}{r}{.53\textwidth}
	\vspace{-0.4cm}    \centering
	\framebox{
		~    \begin{minipage}{.48\textwidth}
			
			\textbf{Parameters:}  $\timeHorizon$, $\partition
			=
			\{\partition_{h,i}\}$\\[.05cm]
			\textbf{Init:}
			Open $\hmax$ times cell $\partition_{0,1}$.  \\
			 $\hmax\gets \left\lfloor\frac{\timeHorizon}{2(\bar\log\timeHorizon+1)^2}\right\rfloor\CommaBin$ $\pmax\gets \left\lfloor\log\hmax\right\rfloor$.          \\[.45cm]
			\textbf{For} $h=1$ to $\hmax$\vspace{.05cm}  \textit{\textbf{\textcolor{gray}{$\hfill\blacktriangleleft$  Exploration $\blacktriangleright$}}}
			
			~~        \textbf{For} $m=1$  to $\left\lfloor\hmax/h\right\rfloor$\vspace{.05cm}
			
			~~~~\, Open $\left\lfloor \frac{\hmax}{hm}\right\rfloor$  times the non-opened    \\         \phantom{aaaaaa}  
			cell $\partition_{h,i}$ with  the   highest values  $\hat f_{h,i}$
			\\  \phantom{aaaaaa}   and given that  $\pullsNumber_{h,i}\geq \left\lfloor \frac{\hmax}{hm}\right\rfloor\cdot$\\[.45cm]
			\textbf{For} $p= 0$ to $\pmax$ \textit{\textbf{\textcolor{gray}{\hfill $\blacktriangleleft$  Cross-validation $\blacktriangleright$}}} \vspace{0.05cm} 
			 
			$\quad$ \textbf{Evaluate} $\hmax/2$ times the 
			\textit{candidates:}\\
			\phantom{aaaaaa} $x(\timeHorizon,p) \gets \argmax\limits_{\left(h,i\right)\in \tree,\,\pullsNumber_{h,i}\geq 2^{p}} \hat f_{h ,i}$.

			\textbf{Output} $x(\timeHorizon) \gets \argmax\limits_{\{x(\timeHorizon,p),\,p\in[0:\pmax]\}} \hat f\left(x(\timeHorizon,p)\right)$
		\end{minipage}
	}
	\caption{The \StroquOOL algorithm}\label{fig:trotro}
\end{wrapfigure}
In the presence of noise, it is natural to evaluate the cells multiple times, not just one time as in the deterministic case.
The amount of times a cell should be evaluated to differentiate its value from the optimal value of the function depends on the gap between these two values as well as the range of noise. As we do not want to make \emph{any} assumptions on knowing these quantities, our algorithm tries to be robust to any potential values by not making a fixed choice on the number of evaluations. Intuitively, \StroquOOL implicitly uses  modified versions of \SequOOL, denoted \SequOOL{}$(p),$\footnote{Again, this is only for the intuition, the algorithm is \emph{not} a meta-algorithm over \SequOOL{}$(p)$s.} where each cell is evaluated $p$ times, $p\geq 1$, while in \SequOOL $p=1$. On one side, given one instance of \SequOOL{}$(p)$, evaluating more each cells ($p$ large) leads to a better quality of the mean estimates in each cell. On the other side, as a tradeoff, it implies that \SequOOL{}$(p)$ is using  more evaluations per depth and therefore is not able to explore deep depths of the partition. The largest depth explored is now $\cO(n/p)$.
\StroquOOL then \emph{implicitly} performs the same amount of evaluations as it would be performed by $\log n$ instances of \SequOOL{}$(p)$ each with a number of evaluations of $p=2^{p'}$, where we have $p'\in[0:\log n]$. 
The \underline{St(r)o}chastic sequential Online Optimization aLgorithm,
\StroquOOL, is described in Figure~\ref{fig:trotro}. 
Remember that `opening' a cell means `evaluating' its children.
The algorithm opens cells by sequentially diving them deeper and deeper from the root node $h=0$ to a maximal depth of $\hmax$.
At depth $h$, we allocate, in a decreasing fashion, different number of evaluations $\left\lfloor \hmax/(hm)\right\rfloor$  to the  cells with highest value of that depth, with $m$ from $1$ to $\left\lfloor \hmax/h \right\rfloor$. The best cell that has been evaluated at least $\cO(\hmax/h)$ times is opened with $\cO(\hmax/h)$ evaluations, the next best cells that have been evaluated at least $\cO(\hmax/(2h))$ times are opened with $\cO(\hmax/(2h))$ evaluations, the next best cells that have been evaluated at least $\cO(\hmax/(3h))$ times are opened with $\cO(\hmax/(3h))$ evaluations and so on, until some $\cO(\hmax/h)$ next best cells that have been evaluated at least once are opened with one evaluation.  
More precisely,   given, $m$ and $h$, we open, with$\left\lfloor \hmax/(hm)\right\rfloor$ evaluations, the $m$ non-previously-opened cells $\partition_{h,i}$ with highest values  $\hat f_{h,i}$ and given that  $\pullsNumber_{h,i}\geq\left\lfloor \hmax/(hm)\right\rfloor.$
For each $p\in[0:\pmax\triangleq \left\lfloor\log_2\left( \hmax\right)\right\rfloor]$, the candidate output  $x(\timeHorizon,p)$ is the cell with highest estimated value that has been evaluated at least $2^p$ times,
$x(\timeHorizon,p) \triangleq \argmax\limits_{\left(h,i\right)\in \tree,\pullsNumber_{h,i}\geq 2^p} \hat f_{h ,i}$. 
%
%
We set $\hmax\triangleq\left\lfloor \timeHorizon/(2(\bar\log\timeHorizon+1)^2)\right\rfloor\!.$ 
Then, 
\StroquOOL uses less  than $\timeHorizon$ openings, which we detail in Appendix~\ref{app:budstro}. 
%
%
\subsection{Analysis of \StroquOOL}\label{sub:anaStro}
%
%
The proofs of the following statements are given in Appendix~\ref{app:prooflemmastro} and~\ref{app:proofALLmastro}.
For any  $x^\star\!\!,$  
$\depthOp_{h,p}$ is
the depth of the deepest opened node with at least $2^p$ evaluations  containing~$x^\star$ 
at the end of the opening of depth 
$h$ of \StroquOOL.

\begin{restatable}{lemma}{restahstarSto}\label{lem:hstarSto}
	For any global optimum $x^\star$ with associated $(\nu,\rho)$ from Assumption~\ref{as:smooth}, any $C>1$, for any $\delta \in (0,1)$, on  event $\xi_\delta$
 defined in Lemma~\ref{l:event},
	 for any pair $(h,p)$ of depths $h$, and integer $p$ 
	 such that $h \in [\hmax]$, and $p\in[0:\log\left\lfloor\hmax/h\right\rfloor]$, we have that if $b \sqrt{\log(2\timeHorizon^2/\delta)/2^{p+1}} \leq \nu\rho^h$ and if $ \hmax/(4h2^p)
	\geq  C\rho^{-d(\nu,C,\rho)h}$, that
	$\depthOp_{h,p}= h$ with $\depthOp_{0,p}\triangleq 0.$
\end{restatable}
\noindent
Lemma~\ref{lem:hstarSto} gives two conditions so that the cell containing $x^\star$ is opened at depth $h$. This holds if (a) \StroquOOL opens, with $2^p$ evaluations, more cells at depth $h$ than the number of near-optimal cells at depth $h$ ($ \hmax/(4h2^p)
\geq  C\rho^{-d(\nu,C,\rho)h}$) and (b) the $2^p$ evaluations are sufficient to discriminate the empirical average of near-optimal cells from the empirical average of sub-optimal cells ($b \sqrt{\log(2\timeHorizon^2/\delta)/2^p} \leq \nu\rho^h$).
%
%
To state the next theorems, we introduce  $\tilde h$ a positive real number satisfying
$(\hmax\nu^2\rho^{2\tilde h})/(4\tilde hb^2\log(2\timeHorizon^{2}/\delta))=C\rho^{-d\tilde h}.$  
We have
\[
\noindent
\tilde h =    \frac{1}{(d+2)\log(1/\rho)}\log\left(\frac{\bar n}{\log \bar n } \right)+o(1) \quad \text{with} \quad \quad\bar n \triangleq \frac{\nu^2\hmax (d+2)\log(1/\rho)}{4Cb^2\log(2\timeHorizon^{2}/\delta)}\cdot
\]
The quantity $\tilde h$ gives the depth of the deepest cell opened by \StroquOOL that contains $x^\star$ with high probability. Consequently,  $\tilde h$ also lets us characterize for which regime of the noise range $b$ we recover results similar to the loss for the deterministic case.
Discriminating on the noise regime, we now state our results, Theorem~\ref{th:highnoise} for a high noise and 
Theorem~\ref{th:lownoise} for a low one.
%
\begin{restatable}{theorem}{restathhighnoise}{\textcolor{gray}{\normalfont\textbf{High-noise regime\ }}}\label{th:highnoise}
	After $\timeHorizon$ rounds, for any function $f$, a global optimum $x^\star$ with associated $(\nu,\rho)$, $C>1$, and near-optimality dimension simply denoted $d=d(\nu,C,\rho)$,   with probability at least $1-\delta$,   if $b\geq\nu\rho^{\tilde h}/\sqrt{\log(2\timeHorizon^2/\delta)},$ the simple regret of \StroquOOL obeys
	%
	%
	\[r_\timeHorizon
	~\leq~
	\nu\rho^{ 
		\frac{1}{(d+2)\log(1/\rho)}\lambertW\left(
		\left\lfloor\frac{\timeHorizon}{2(\log_2\timeHorizon+1)^2}\right\rfloor
		\frac{ (d+2)\log(1/\rho)\nu^2}{4Cb^2\log(2\timeHorizon^2/\delta)}\right)}
	+2b\sqrt{\log(2\timeHorizon^2/\delta)\bigg/\left\lfloor\frac{\timeHorizon}{2(\log_2\timeHorizon+1)^2}\right\rfloor}\cdot
	\]
\end{restatable}
\begin{corollary}\label{c:readableHigh}
	With the assumptions of Theorem~\ref{th:highnoise} and $
	\bar n>e$, 
	\[
	r_\timeHorizon
	\leq
	\nu \left(\frac{\log\bar\timeHorizon}{\bar n}\right)^{ \frac{1}{d+2}}+2b\sqrt{\frac{18\log(2\timeHorizon^2/\delta)}{2\left\lfloor\frac{\timeHorizon}{2(\log_2\timeHorizon+1)^2}\right\rfloor}}\!\CommaBin
	\text{ ~where~ } \bar n \triangleq 
	\left\lfloor\frac{\timeHorizon/2}{(\log_2\timeHorizon+1)^2}\right\rfloor\frac{ (d+2)\log(1/\rho)\nu^2}{4Cb^2\log(2\timeHorizon^2/\delta)}\cdot
	\]
	%
\end{corollary}
%
\begin{restatable}{theorem}{restathlownoise}{\textcolor{gray}{\normalfont\textbf{Low-noise regime\ }}}\label{th:lownoise}
	After $\timeHorizon$ rounds, for any function $f$ and one of its global optimum $x^\star$ with associated $(\nu,\rho)$, any $C>1$, and near-optimality dimension simply denoted $d=d(\nu,C,\rho)$, with probability at least $1-\delta$, if $b\leq\nu\rho^{\tilde h}/\sqrt{\log(2\timeHorizon^2/\delta)},$ the simple regret of \StroquOOL obeys 
	%
	%
	\begin{align*}
	&\text{\textcolor{black}{\textbullet~  If $d=0$,} }
	~~ r_\timeHorizon
	~\leq~ 3\nu\rho^{\frac{1}{4C}\left\lfloor\frac{\timeHorizon/2}{(\log_2(\timeHorizon)+1)^2}\right\rfloor}. 
	&&\text{\textcolor{black}{\textbullet~  If $d>0$,} }
	~~ r_\timeHorizon
	~\leq~
	3\nu e^{ 
		-\frac{1}{d}\lambertW\left(\left\lfloor\frac{\timeHorizon/2}{(\log_2\timeHorizon+1)^2}\right\rfloor\frac{ d\log{\frac{1}{\rho}}}{4C}\right)}.
	\end{align*}
\end{restatable}
This results also hold for the deterministic feedback case, $b=0$, with probability $1$.
\begin{corollary}\label{c:readableLow}
	With the assumptions of Theorem~\ref{th:lownoise}, 
	if $d>0$, then\vspace{-0.05cm}
	\[
	r_\timeHorizon
	~\leq~
	3\nu  \left(\frac{\log\left(\tilde n\right)}{\tilde n}\right)^{ \frac{1}{d}} \text{ ~~~~with~~~~ } \tilde n \triangleq 
	\left\lfloor\frac{\timeHorizon/2}{(\log_2\timeHorizon+1)^2}\right\rfloor\frac{ d\log(1/\rho)}{4C}
	\text{ ~and~ }
	\tilde n>e.
	\]\vspace{-0.5cm}
\end{corollary}
\subsection{Discussion for the stochastic feedback\vspace{-0.1cm}}
\paragraph{Worst-case comparison to \POO and \StoSOO} \textcolor{gray}{\textbf{When $b$ is large and known\ }}
\StroquOOL is an algorithm designed for the noisy feedback while adapting to the smoothness of the function. Therefore, it can be directly compared to  \POO and \StoSOO that both tackle the same problem.
The results for \StroquOOL, like the ones for \POO, hold for $d\geq 0$, while the theoretical guarantees of \StoSOO are only for the case $d=0$.
The general rate of \StroquOOL in Corollary~\ref{c:readableHigh}\,\footnote{Note that the second term in our bound has at most the same rate as the first one.} is similar to the ones of \POO (for $d\geq 0$) and \StoSOO (for $d=0$) as their loss is $\tcO(\timeHorizon^{-1/(d+2)})$.
More precisely, looking at the log factors, we can first notice an improvement over  \StoSOO when $d=0$. We have  
$r^{\StroquOOL}_{\timeHorizon}=
\cO(\log^{3/2}(\timeHorizon)/\sqrt{\timeHorizon})
\leq
r^{\StoSOO}_{\timeHorizon}=
\cO(\log^{2}\timeHorizon/\sqrt{\timeHorizon})$.   Comparing with \POO,  we obtain a worse logarithmic factor, as 
$r^{\POO}_{\timeHorizon}=
\cO( (\log^2(\timeHorizon)/\timeHorizon))^{1/(d+2)})
\leq r^{\StroquOOL}_{\timeHorizon}=\cO(((\log^3\timeHorizon)/\timeHorizon)^{1/(d+2)} )$. Despite having this (theoretically) slightly worse logarithmic factor compared to \POO,  \StroquOOL has two nice new features. First, our algorithm is conceptually simple, parameter-free, and does not need to call a sub-algorithm: \POO repetitively calls different instances of \HOO which makes it a heavy meta-algorithm. Second, our algorithm, as we detail next,
naturally adapts to low noise and, even more,  recovers the rates of \SequOOL in the deterministic case, leading to exponentially decreasing loss when $d=0$. We do not know if  the extra  logarithmic factor for \StroquOOL as compared to \POO to  is the unavoidable price to pay to obtain an adaptation to the deterministic feedback case.\vspace{-0.15cm}
%
\paragraph{Comparison to \HOO}
\HOO is also  designed for the noisy optimization setting. \HOO \textit{needs to know the smoothness} of $f$, i.e., $(\nu,\rho)$ are input parameters of \HOO. Using this extra knowledge \HOO is only able to improve the logarithmic factor  to achieve a regret of  $r^{\HOO}_{\timeHorizon}=
\cO( (\log(\timeHorizon)/\timeHorizon)^{1/(d+2)} )$. 
%
\paragraph{Adaptation to the range of the noise $b$ without a prior knowledge}
A favorable feature of our bound in Corollary~\ref{c:readableHigh} is that it characterizes  how the range of the noise $b$ affects the rate of the regret  for all $d\geq  0$. Effectively, the regret of \StroquOOL scales with
$\left( n/b^2\right)^{-1/(d+2)}$. 
Note that $b$ is any real non-negative number and it is unknown to \StroquOOL.
To achieve this result, and contrarily to  \HOO, \StoSOO, or \POO, we designed \StroquOOL \emph{without using upper-confidence bounds} (UCBs). Indeed, UCB approaches are overly conservative as they use, in the design of their confidence bound, hard-coded (and often overestimated) upper-bound on $b$ that we denote $\tilde b$.
\HOO, \POO, and \StoSOO, would only obtain a similar regret to \StroquOOL, scaling with $b$, when~$b$ is known to them, in with case $\tilde b$ would be set as $\tilde b=b$. In general, UCB approaches have their regret 
scaling with
$( n/\tilde b^2)^{-1/(d+2)}$.
Therefore, the most significant improvement of \StroquOOL over \HOO, \POO, and \StoSOO is expected when $\tilde b \gg  b$.

 %
\paragraph{Adaptation to the deterministic case and $d\!\!=\!\!0$}
When the noise is very low, that is, when $b\leq\nu\rho^{\tilde h}/\sqrt{\log(2\timeHorizon^{2}/\delta)}$, which includes  the deterministic feedback,  in Theorem~\ref{th:lownoise} and Corollary~\ref{c:readableLow}, \StroquOOL recovers the same rate as \DOO and \SequOOL up to logarithmic factors. Remarkably, \StroquOOL obtains an exponentially decreasing regret when $d=0$ while \POO, \StoSOO, or \HOO only guarantee  a regret  of $\tcO(\sqrt{1/n})$ when unaware of the range $b$. Therefore, up to log factors, \StroquOOL achieves naturally the \emph{best of both worlds} without being aware of the nature of the feedback (either stochastic or deterministic). 
Again, if the input noise parameter  $\tilde b \gg  b$ (it is often set to $1$ by default)
this is a behavior that one \emph{cannot} expect from \HOO, \POO, or \StoSOO as they explicitly use confidence intervals based on $\tilde b$.
Finally,  using UCB approaches with empirical estimation of the variance $\hat \sigma^2$ \textit{would not} circumvent this behavior. Indeed, the UCB in such case is typically of the form $\sqrt{\hat\sigma^2/T}+\tilde b/T$~\citep{maurer2009empirical}. Then if $\tilde b \gg b$, the term 
$\tilde b/T$ in the upper confidence bound will force  an overly conservative exploration. This \textit{prevents} having  $e^{-\tilde\Omega(n)}$ when $d=0$ and $b\approx0$.
\vspace{-0.1cm}

 \vspace{-0.1cm}
%
%
%
	%
\section{Experiments}\label{experiements}
We empirically demonstrate how \SequOOL and \StroquOOL adapt to the complexity of the data and compare them to \SOO, \POO, and \HOO.
\begin{figure}
	\center
	\includegraphics[width=0.325\textwidth]{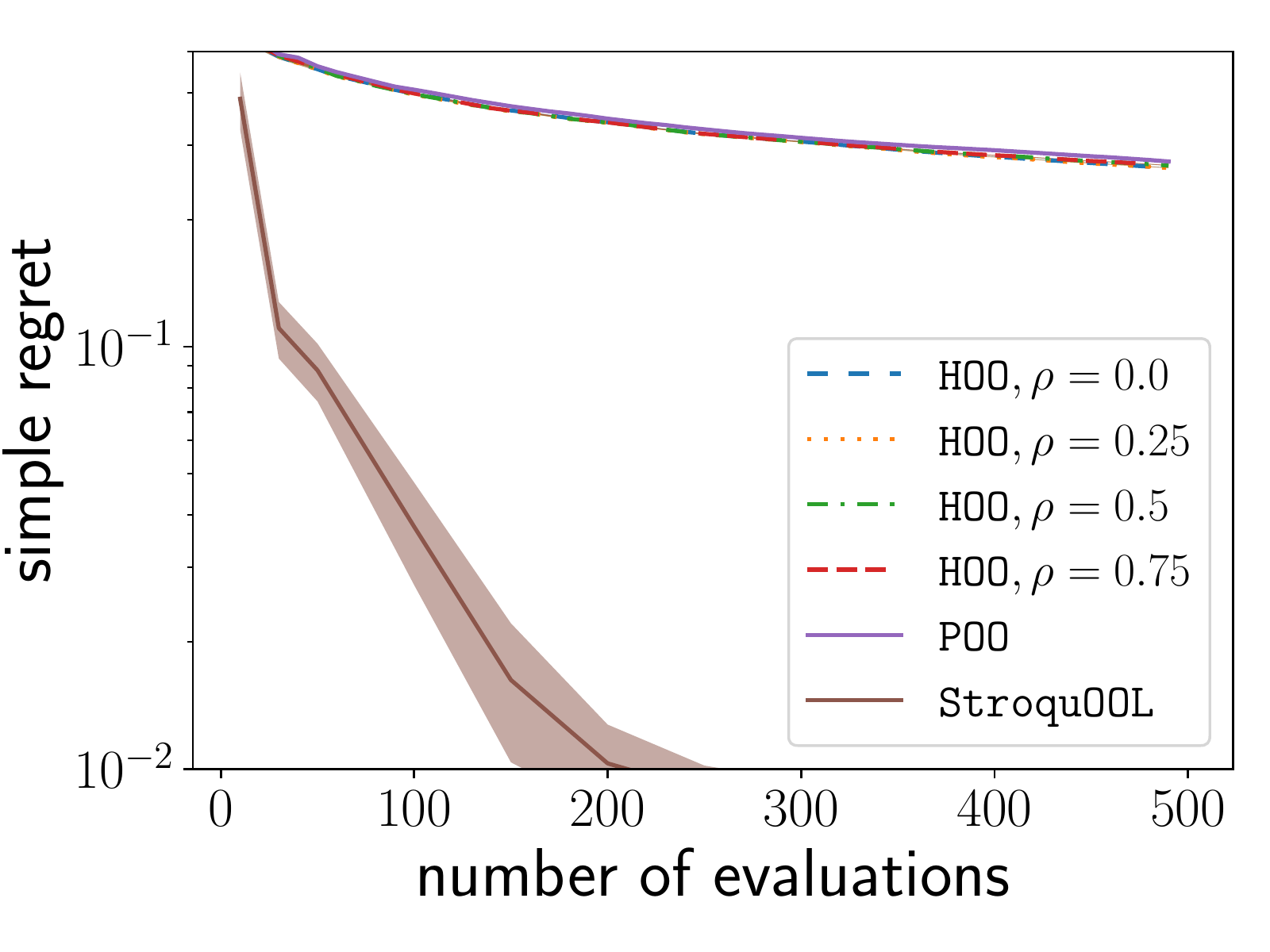}
	\includegraphics[width=0.325\textwidth]{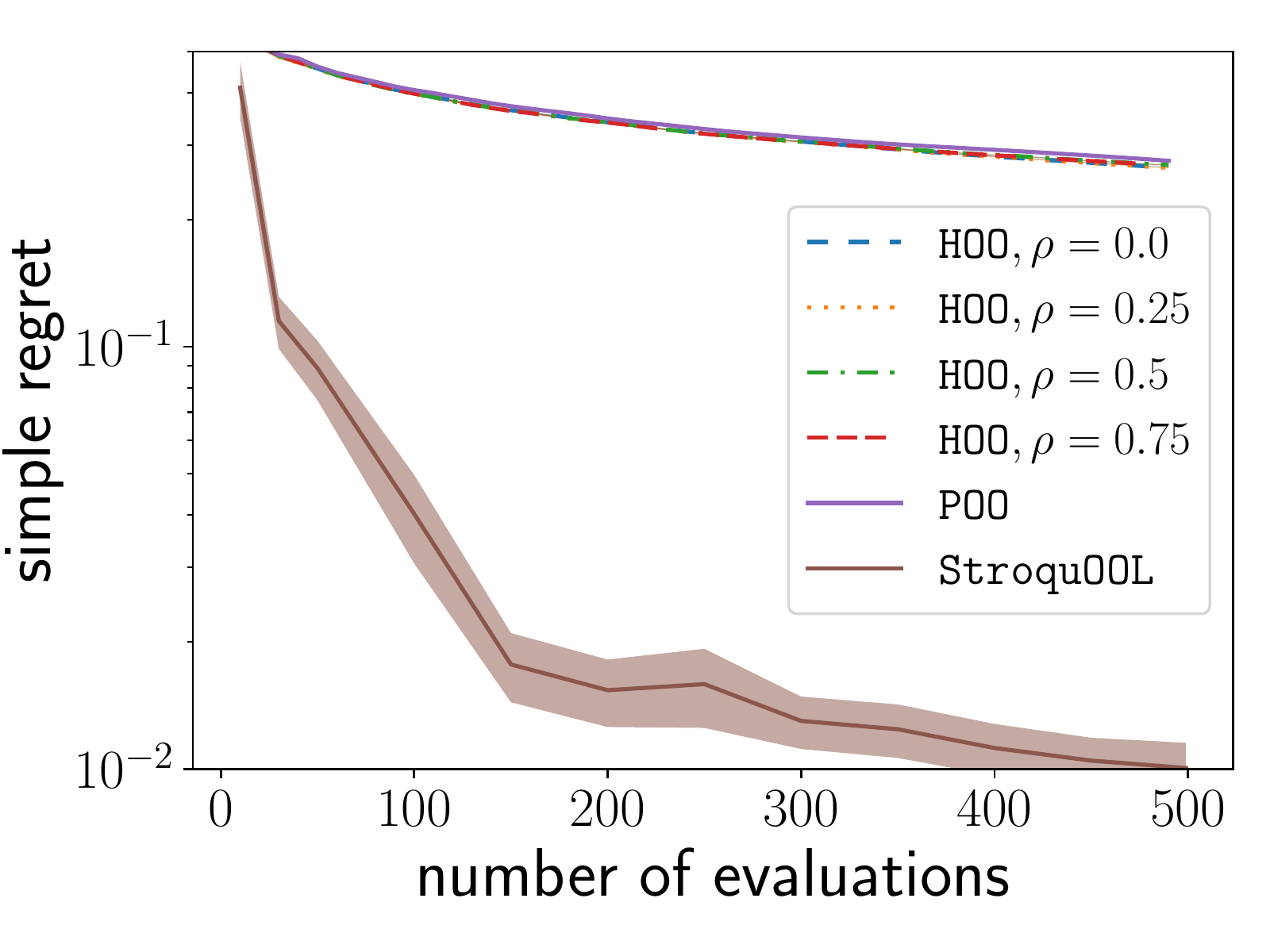}
	\includegraphics[width=0.325\textwidth]{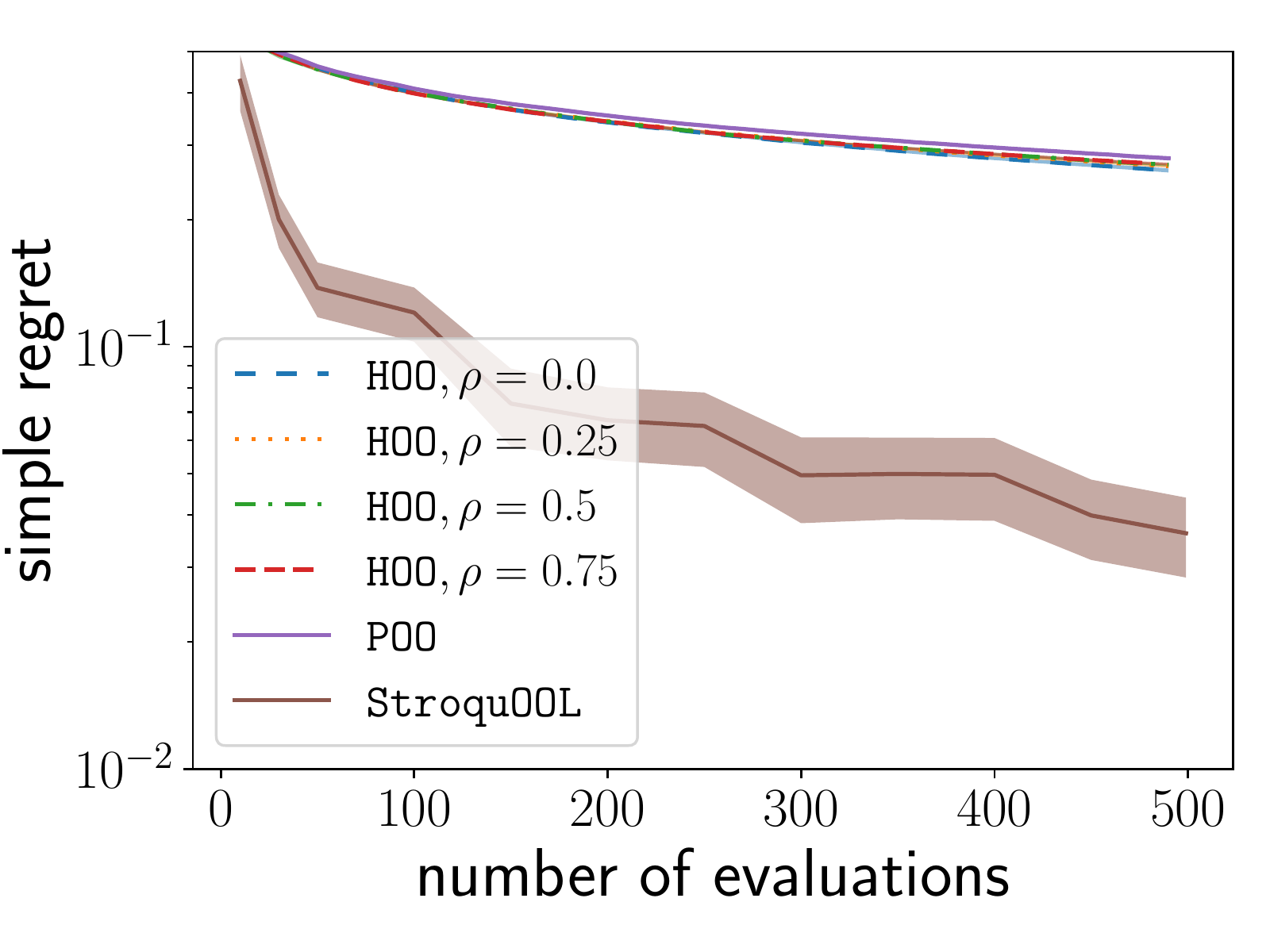}\\
	\includegraphics[width=0.325\textwidth,valign=t]{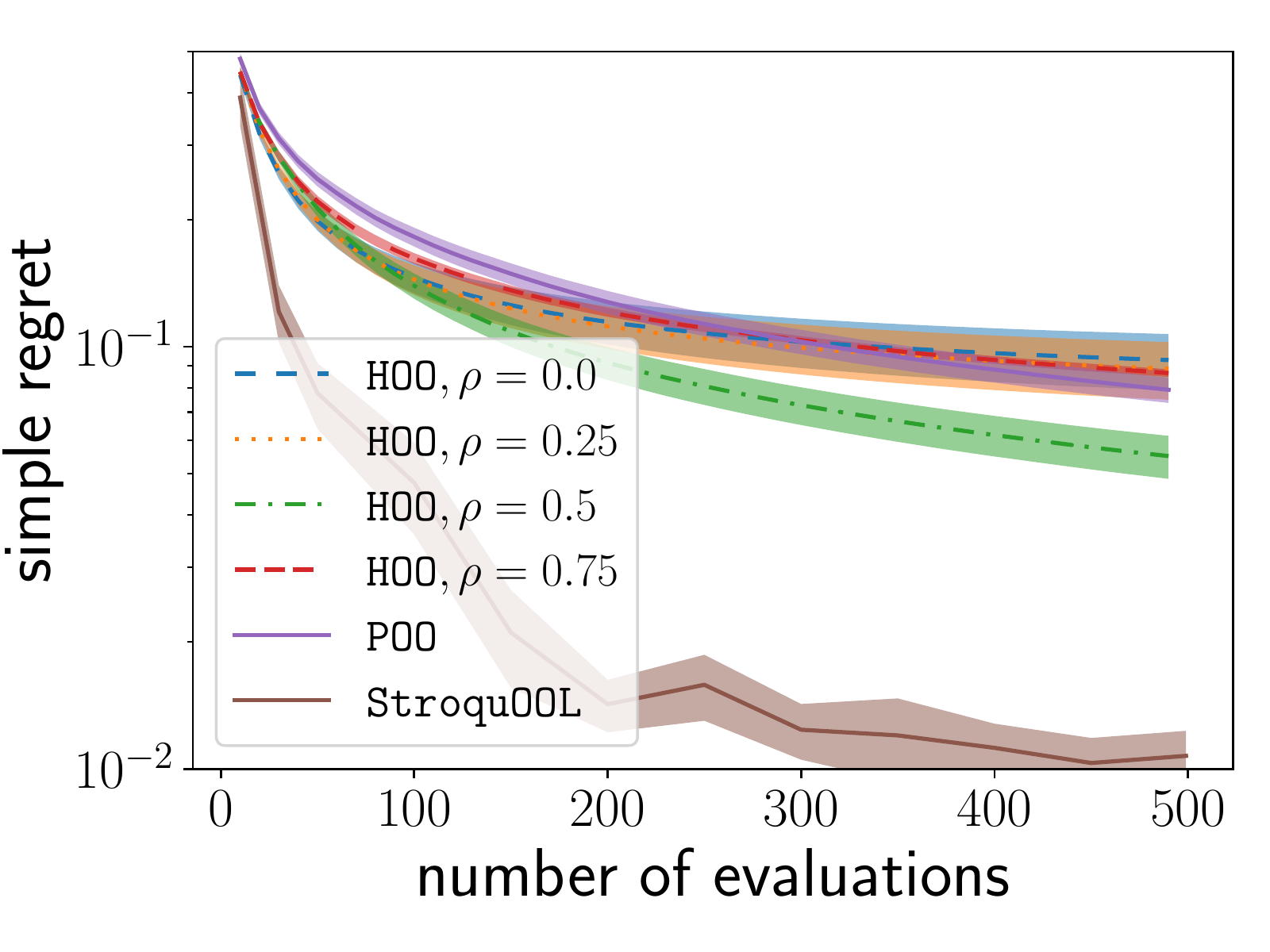}
	\includegraphics[width=0.325\textwidth,valign=t]{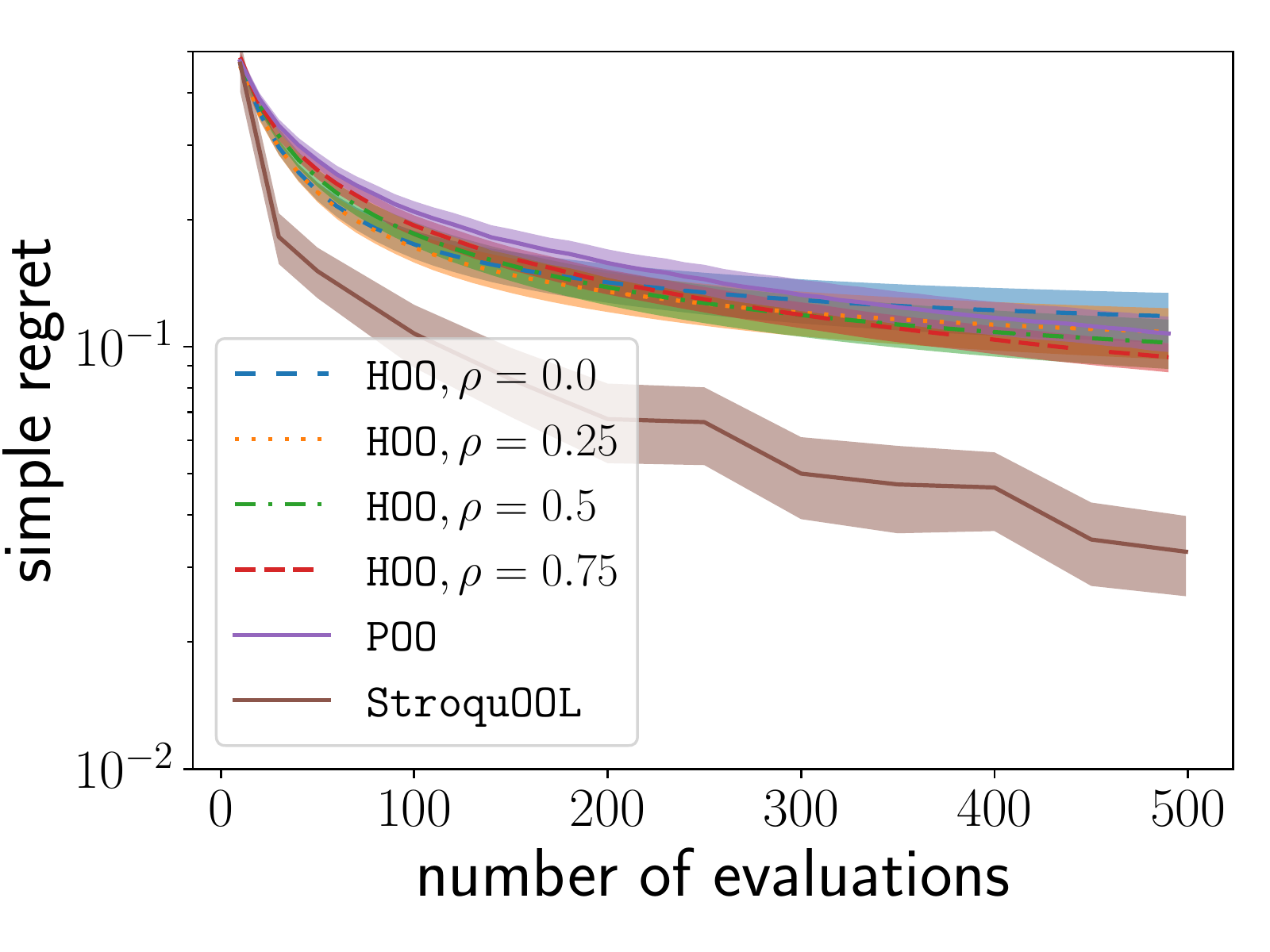}
	~~\includegraphics[width=0.32\textwidth,valign=t]{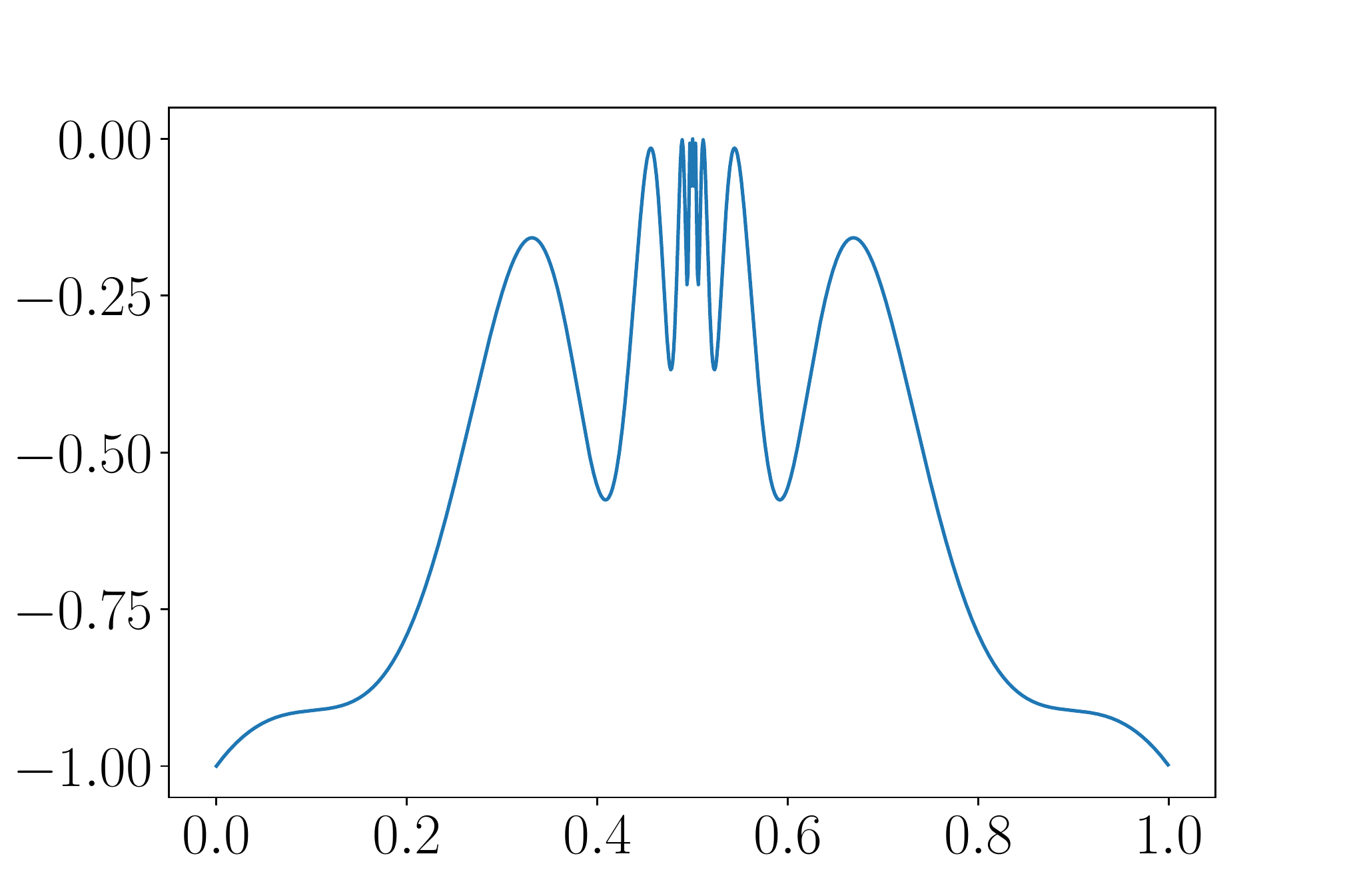}
	
	\caption{\emph{Bottom right: } \textbf{Wrapped-sine} function ($d>0$). The true range of the noise~$b$ and the range 	used by \HOO and \POO is $\tilde b$. 
		\emph{Top:} $b=0,\tilde b=1$ left --- $b=0.1,\tilde b=1$ middle --- $b=\tilde b=1$ right.
		\emph{Bottom:} $b=\tilde b=0.1$ left --- $b=1,\tilde b=0.1$ middle.  }
	\label{Sin}
\end{figure}
We use two functions used by prior work as testbeds for optimization of difficult function without the knowledge of smoothness.
The first one is the \textbf{wrapped-sine} function ($S(x),$ \citealp{Grill15BB}, Figure~\ref{Sin}, bottom right) with $S(x) \triangleq \frac12(\sin(\pi\log_2 (2|x-\frac12|))+1)((2|x-\frac12|)^{-\log .8}-(2|x-\frac12|)^{-\log .3})-(2|x-\frac12|)^{-\log .8}$. This function has $d>0$ for the standard partitioning \citep{Grill15BB}.
	 The second is the \textbf{garland} function ($G(x),$ \citealp{Valko88SS}, Figure~\ref{Gar}, bottom right) with $G(x)\triangleq4x(1-x)(\frac{3}{4}+\frac{1}{4}(1-\sqrt{|\sin(60x)|}))$. Function~$G$ has $d=0$ for the standard partitioning~\citep{Valko88SS}.
	 Both functions are in one dimension, $\dom = \Real$. 
	 Our algorithms work in any dimension, but, with the current computational power available, they would \textit{not} scale beyond a thousand dimensions. 
\paragraph{{\StroquOOL} outperforms \POO and \HOO and adapts  to lower noise.}
In Figure~\ref{Sin}, we report the results of \StroquOOL, \POO, and \HOO for different values of $\rho$. As detailed in the caption, we vary the range of noise $b$ and the range of noise $\tilde b$ used by \HOO and \POO. In all our experiments, \textit{\StroquOOL outperforms \POO and \HOO}. \StroquOOL adapts to low noise, its performance improves when $b$ diminishes. To see that, compare top-left ($b=0$), top-middle ($b=.1$), and top-right ($b=1$) subfigures. On the other hand, \POO and \HOO do not naturally adapt to the range of the noise: For a given parameter $\tilde b=1$, the performance is unchanged when the range of the real noise varies as seen by comparing again top-left ($b=0$), top-middle ($b=.1$), and top-right ($b=1$). However, note that \POO and \HOO \emph{can} adapt to noise and perform empirically well if they have a good estimate of the range $b=\tilde b$ as in bottom-left, or if they underestimate the range of the noise, $\tilde b\ll b$, as in bottom-middle. In Figure~\ref{GarHOO}, we report similar results on the garland function. Finally, \StroquOOL demonstrates its adaptation to both worlds in Figure~\ref{Gar} (left), where it achieves exponential decreasing loss in the case $d=0$ and deterministic feedback. 
%
\paragraph{Regrets of {\SequOOL}  and {\StroquOOL}  have exponential decay when $d=0$.} In Figure~\ref{Gar}, we test in the deterministic feedback case with \SequOOL, \StroquOOL, \SOO and the uniform strategy on the garland function (left) and the wrap-sine function (middle). Interestingly, for the garland function, where $d=0$, \SequOOL outperforms \SOO and displays a truly exponential regret decay (y-axis is in log scale). \SOO appears to have the regret of $e^{-\sqrt{n}}$. \StroquOOL which is expected to have a regret $e^{-n/\log^2n}$ lags behind \SOO. Indeed, $n/\log^2n$ exceeds $\sqrt{n}$ for $n>10000$, for which the result is beyond the numerical precision. In Figure~\ref{Gar} (middle), we used the  wrapped-sine. While all algorithms have similar theoretical guaranties since here $d>0$, \SOO outperforms the other algorithms.
\begin{figure}
	\center
	\includegraphics[width=0.32\textwidth,valign=t]{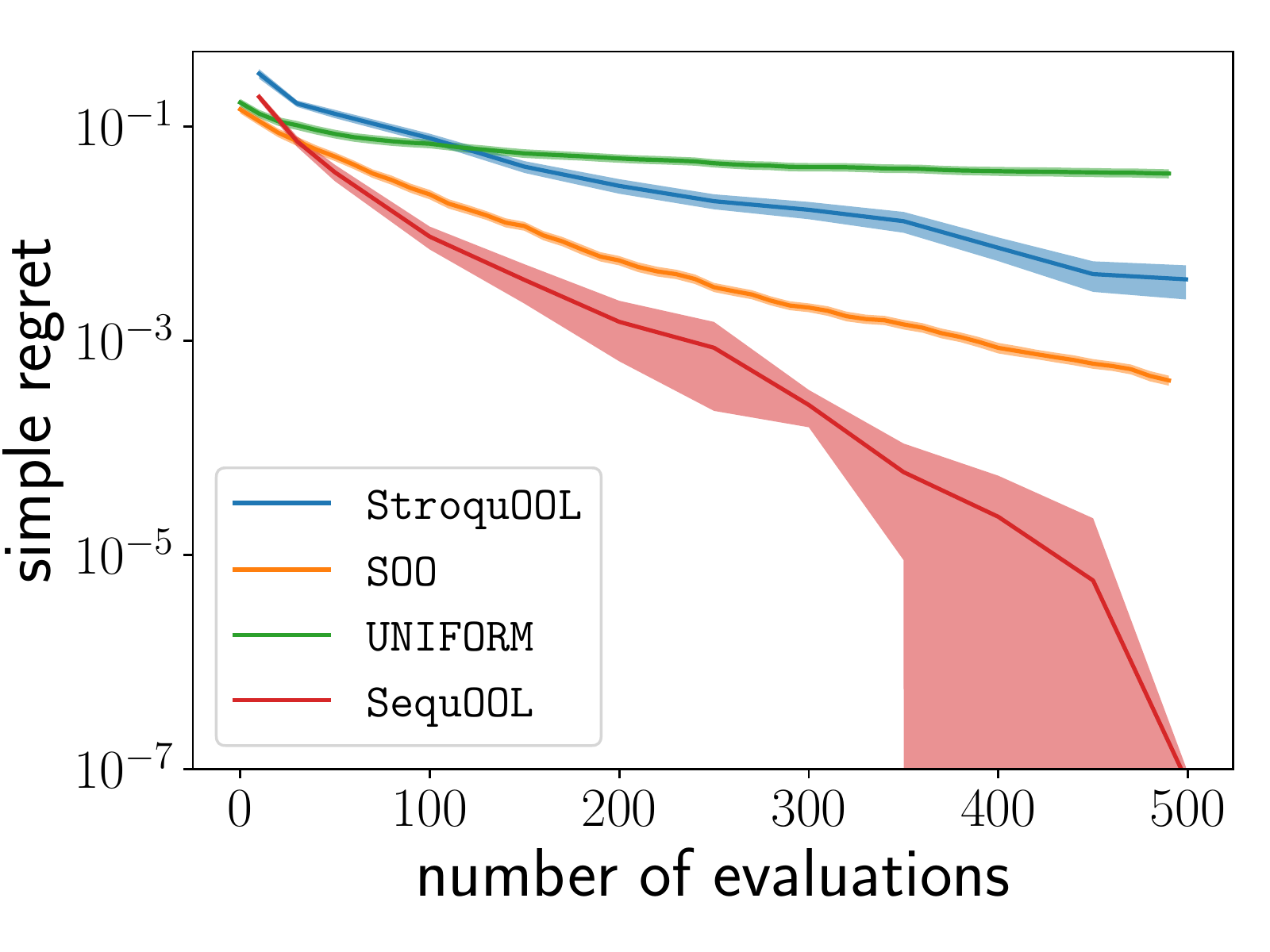}
	\includegraphics[width=0.32\textwidth,valign=t]{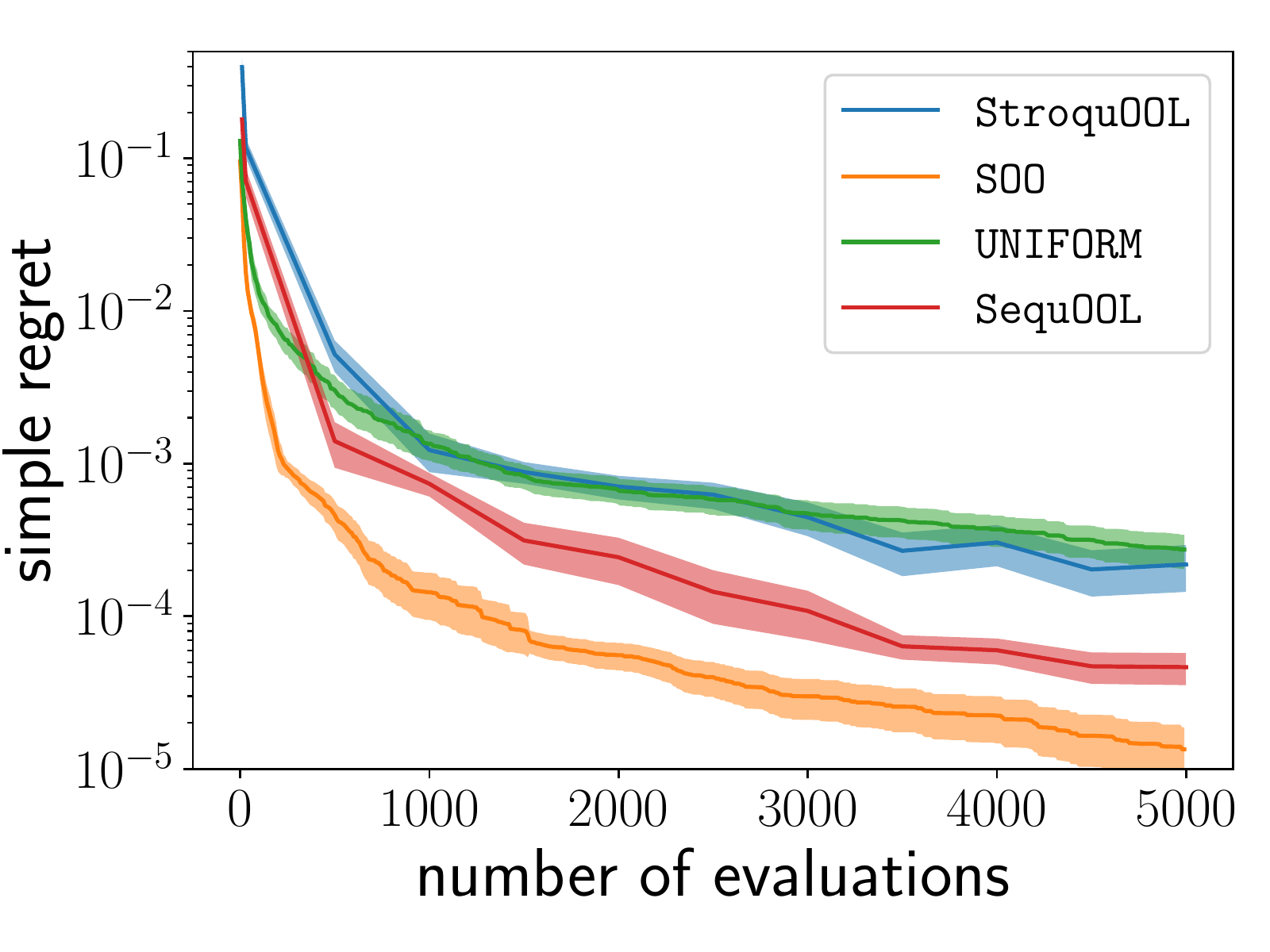}
	~~~	\includegraphics[width=0.31\textwidth,valign=t]{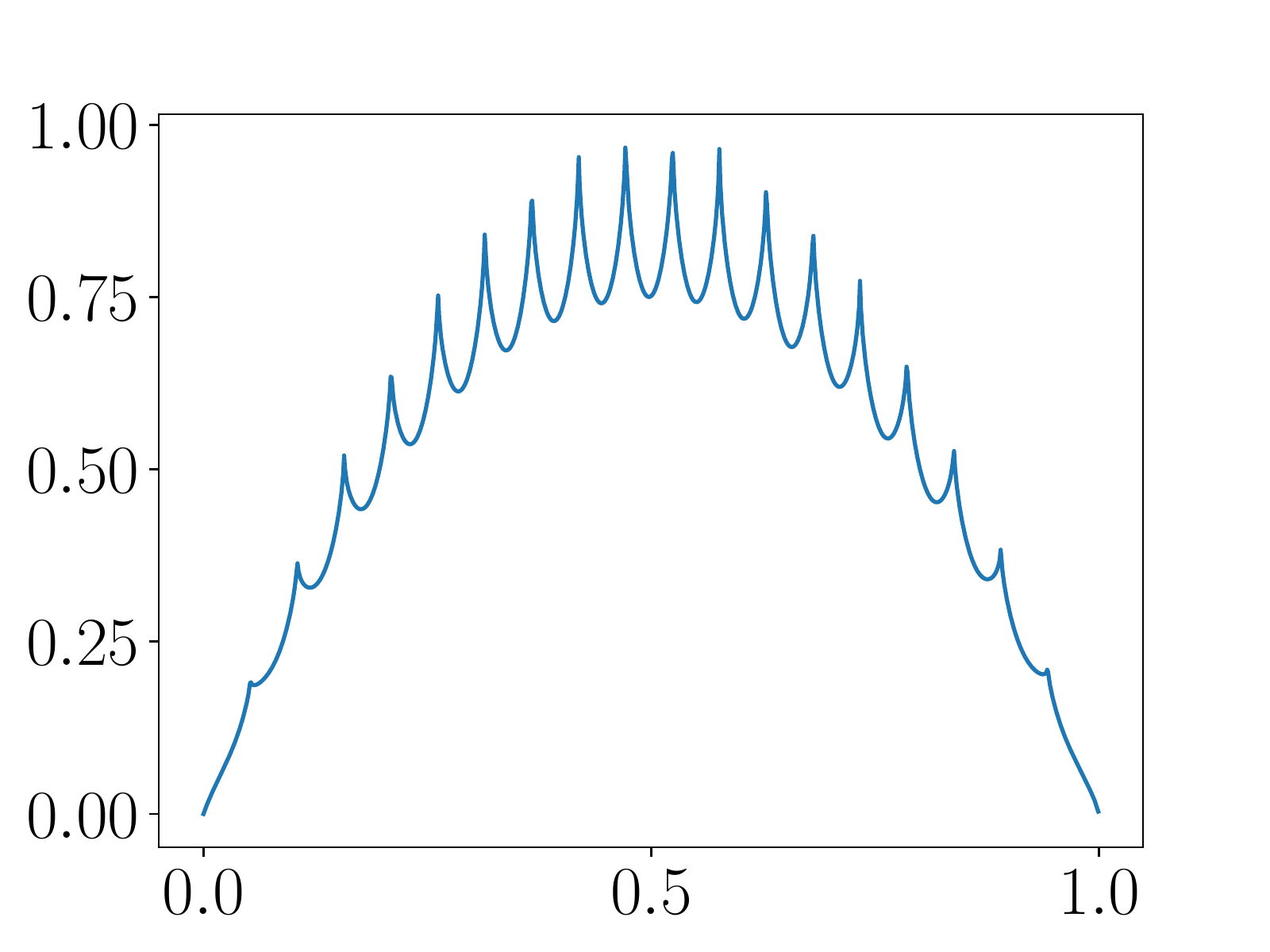}
	
	\caption{\emph{Left \& center: }Deterministic feedback. \emph{Right: }\textbf{Garland} function for which~$d=0$.}
	\label{Gar}
\end{figure}
\begin{figure}
	\center
	\includegraphics[width=0.325\textwidth,valign=t]{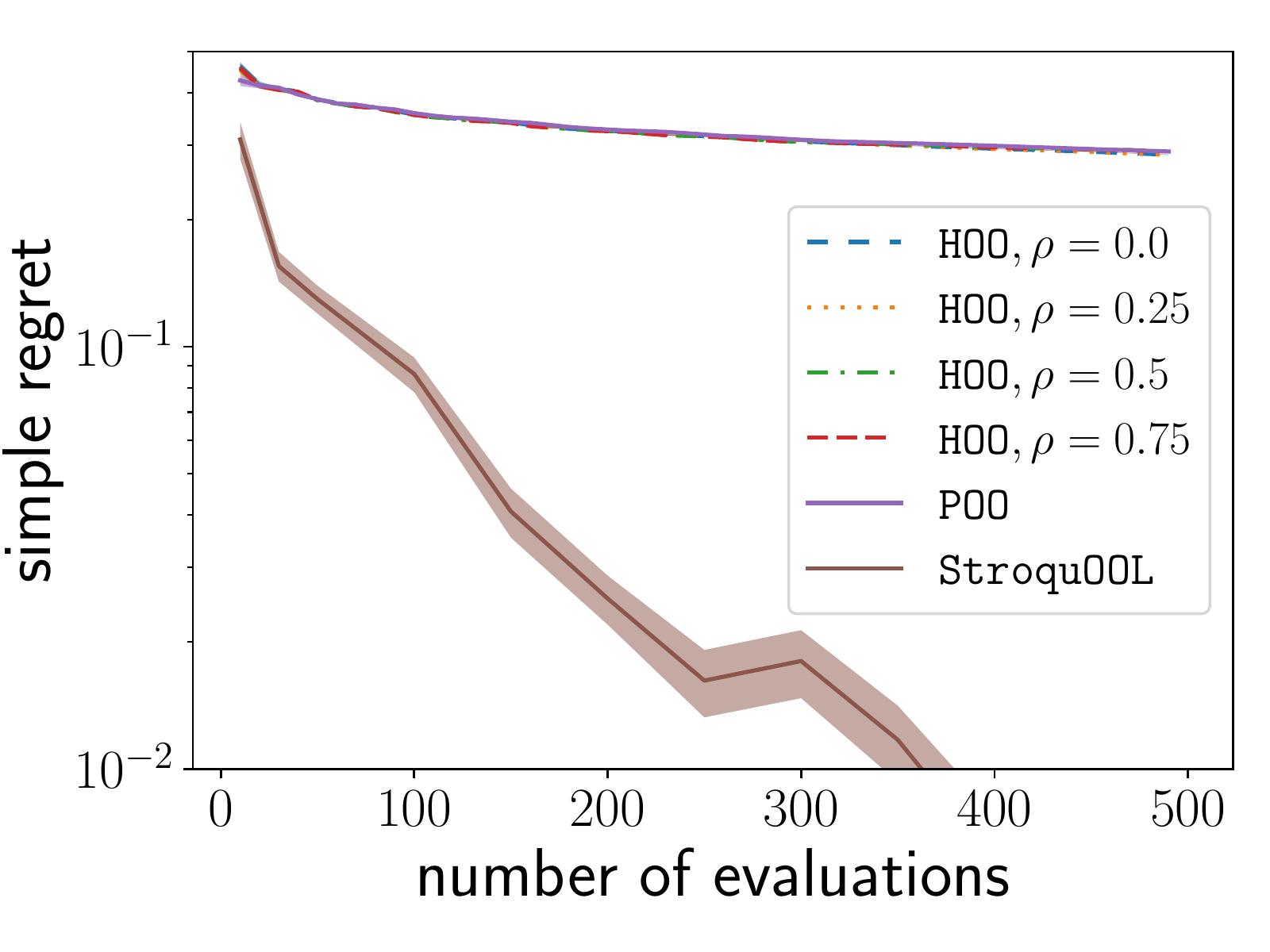}
	\includegraphics[width=0.325\textwidth,valign=t]{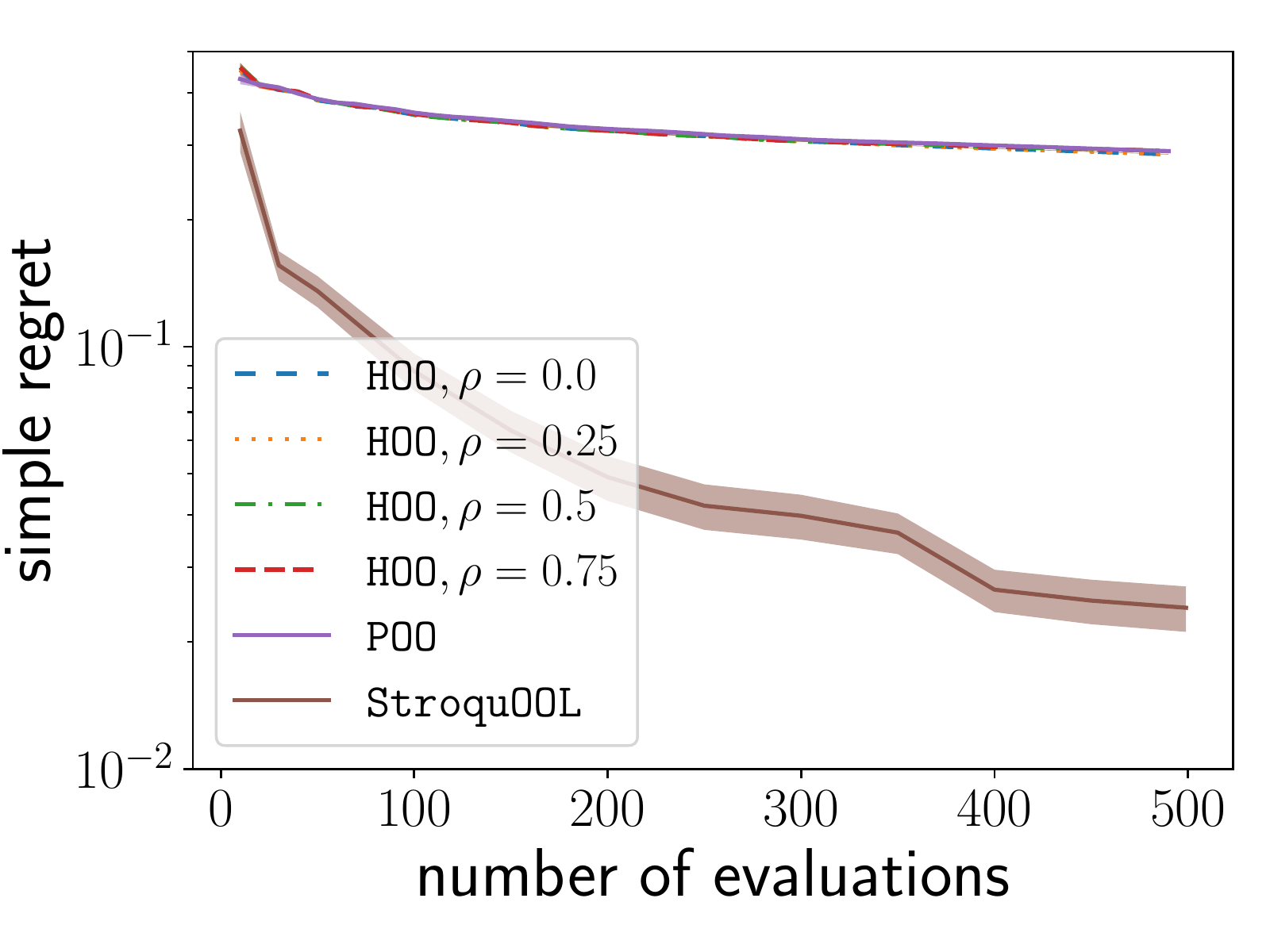}
	\includegraphics[width=0.325\textwidth,valign=t]{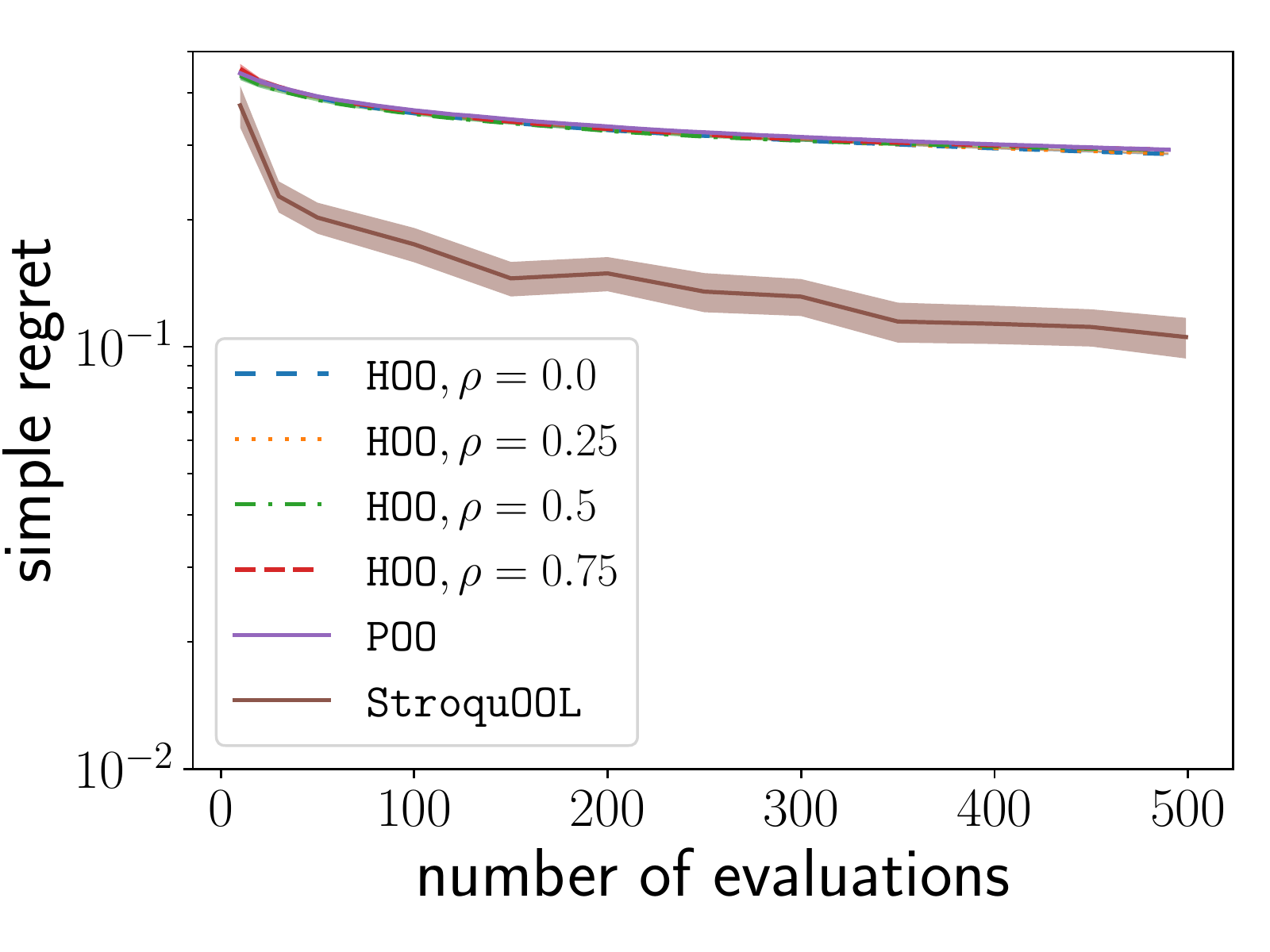}\\
	\includegraphics[width=0.325\textwidth,valign=t]{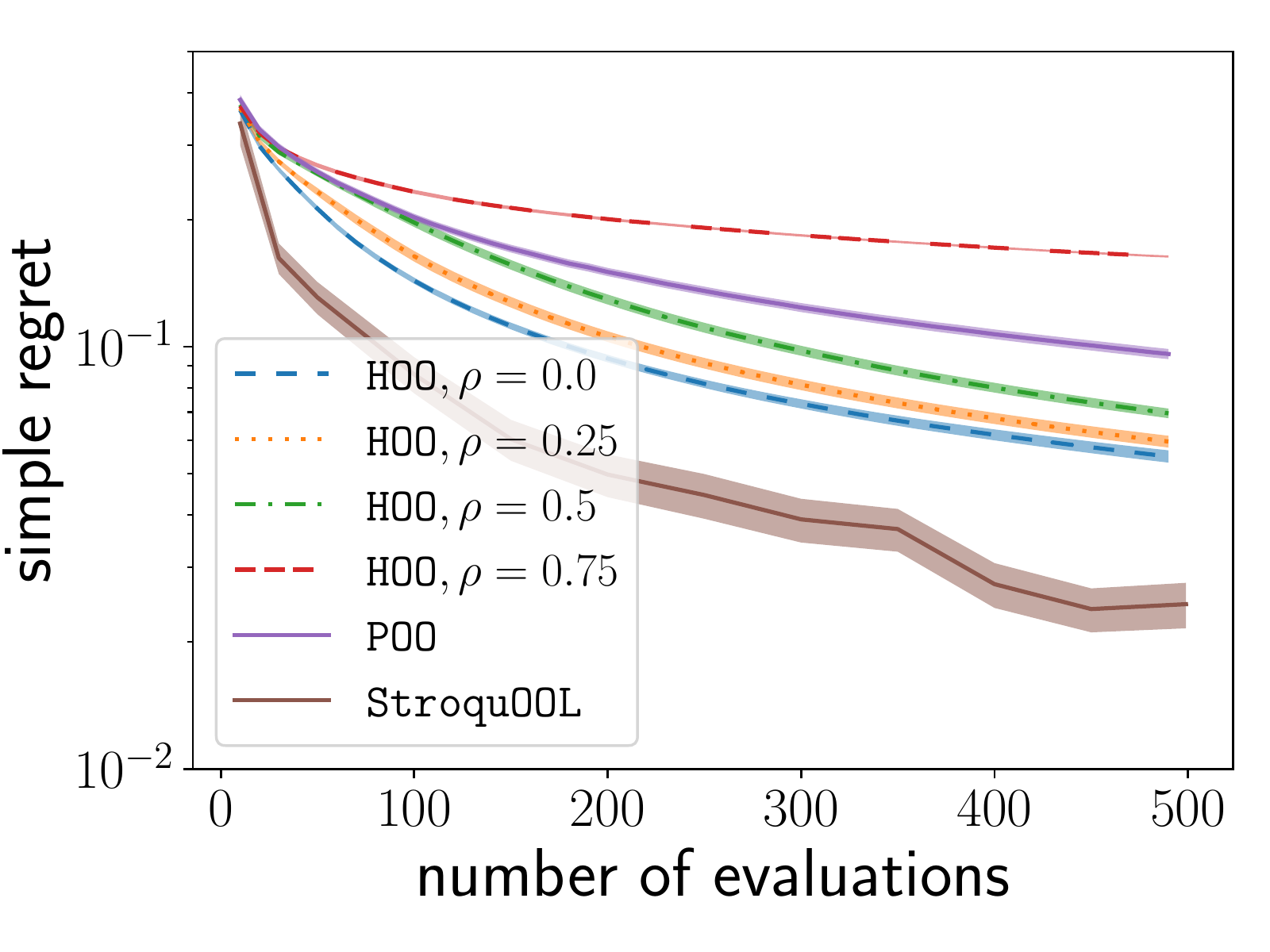}
	\includegraphics[width=0.325\textwidth,valign=t]{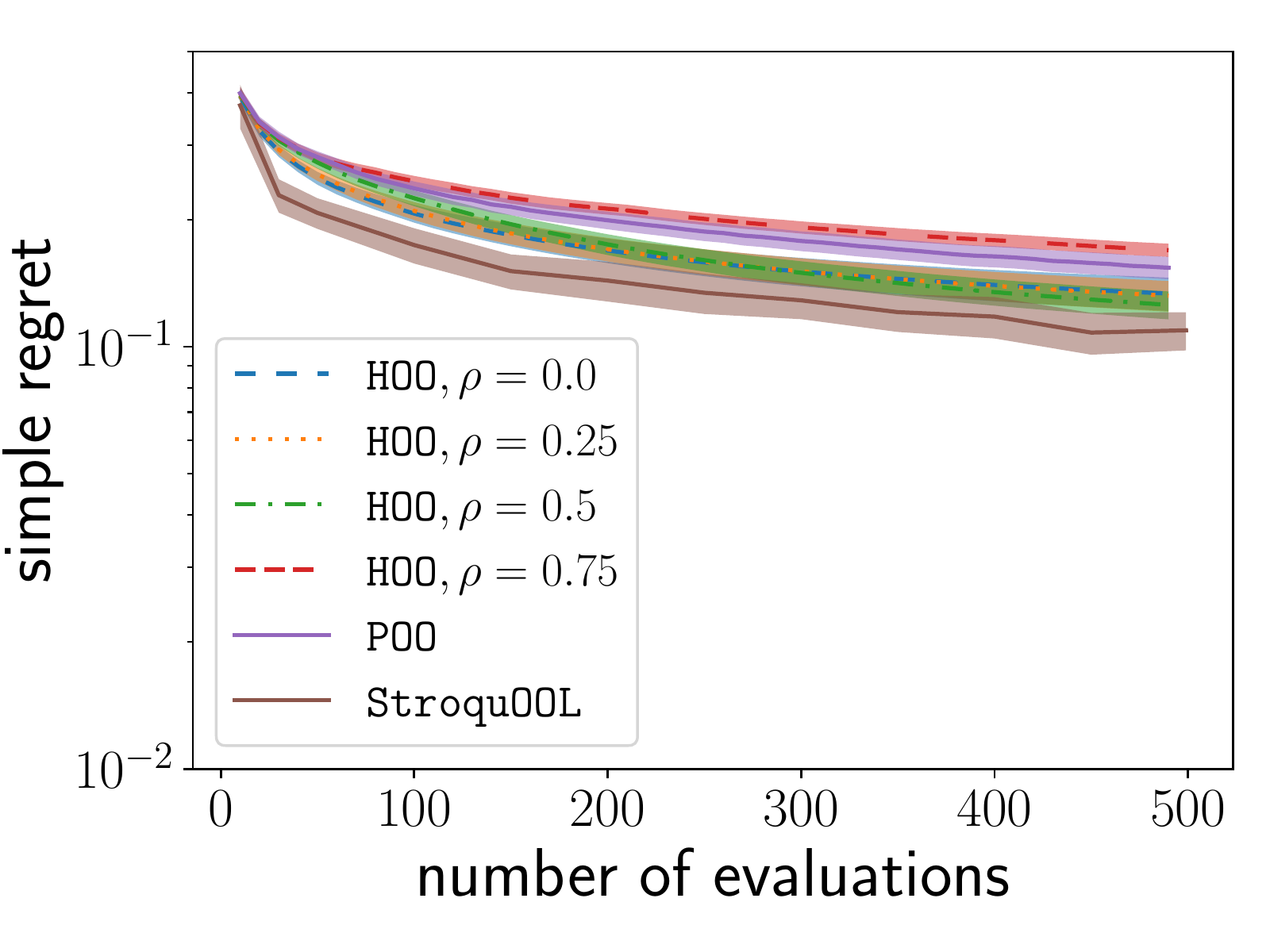}
	~~~    \includegraphics[width=0.3\textwidth,valign=t]{./Garland.pdf}
	\caption{\textbf{Garland} function: The true range of the noise is $b$ and the range of noise used by \HOO and \POO is $\tilde b$ and they are set as
		top:  $b=0,\tilde b=1$ left --- $b=0.1,\tilde b=1$ middle --- $b=1,\tilde b=1$ right,
		bottom: $b=0.1,\tilde b=0.1$ left --- $b=1,\tilde b=0.1$ middle.}
	\label{GarHOO}
\end{figure}

	A more thorough empirical study is desired. Especially  we would like to see how our methods compare with state-of-the-art black-box GO approaches~\citep{pinter2018difficult,pinter2018globally,strongin2000global,sergeyev2013introduction,sergeyev2017deterministic,sergeyev2006global,sergeyev1998global,lera2010information,kvasov2012lipschitz,lera2015deterministic,kvasov2015deterministic}.

	
		\subsection*{Acknowledgements}
	We would like to thank Jean-Bastien Grill for his code and C\^ome Fiegel for helpful discussions and proof reading. We gratefully acknowledge the support of the NSF through grant IIS-1619362 and of the Australian Research Council through an Australian Laureate Fellowship (FL110100281) and through the Australian Research Council Centre of Excellence for Mathematical and Statistical Frontiers (ACEMS).
	The research presented was also supported by European CHIST-ERA project DELTA, French Ministry of
	Higher Education and Research, Nord-Pas-de-Calais Regional Council,
	Inria and Otto-von-Guericke-Universit\"at Magdeburg associated-team north-european project Allocate, and French National Research Agency projects ExTra-Learn (n.ANR-14-CE24-0010-01) and BoB (n.ANR-16-CE23-0003). This research has also benefited from the support of the FMJH Program PGMO and from the support to this program from Criteo.

	\bibliography{library,biblio}
	
	\newpage

	\appendix
\section{Regret analysis of \SequOOL for deterministic feedback}
\label{app:firstone}
\setcounter{scratchcounter}{\value{theorem}}\setcounter{theorem}{\the\numexpr\getrefnumber{lem:hstar}-1}
\restalemhstar*
\setcounter{theorem}{\the\numexpr\value{scratchcounter}}
\begin{proof}
	We prove Lemma~\ref{lem:hstar} by induction in the following sense. For a given $h$, we assume the hypotheses of the lemma for that $h$ are true and we prove by induction that $\depthOp_{h'}= h'$ for $h'\in[h]$. \\[.1cm]    
	$1^\circ$ For $h'= 0$, we trivially have $\depthOp_{h'} \geq 0$.  \\    
	$2^\circ$  Now consider $h'>0$ and assume $\depthOp_{h'-1}= h'-1$ with the objective to prove $\depthOp_{h'}= h'$.
	Therefore, at the end of the processing of depth $h'-1$, during which we were opening the cells of depth $h'-1$ we managed to open the cell $(h'-1, i_{h'-1}^\star)$
	the optimal node of depth $h'-1$ (i.e., such that
	$x^\star \in \mathcal P_{h'-1,i^\star_{h'-1}})$.
	During phase $h'$, the $\left\lfloor \frac{\hmax}{h'}\right\rfloor$ cells from $\left\{\partition_{h',i}\right\}_i$ with highest values $\left\{f_{h',i}\right\}_i$ are opened. 
	For the purpose of contradiction, let us assume $\depthOp_{h'}= h'-1$ that is $\partition_{h',i_h^\star}$ is not one of them. This would mean that there exist at least 
	$\left\lfloor \frac{\hmax}{h'}\right\rfloor$ cells from $\left\{\partition_{h',i}\right\}_i$, distinct from $\partition_{h',i_h^\star}$, satisfying 
	$f_{h',i}
	\geq
	f_{h',i^\star_h}$. As $f_{h',i^\star} \geq f(x^\star) -\nu\rho^{h'}$ by Assumption~\ref{as:smooth}, this means we have $ \mathcal N_{h'}(3\nu\rho^{h'})\geq \left\lfloor\frac{\hmax}{h'}\right\rfloor+1$ (the $+1$ is for $\partition_{h',i_h^\star}$). As $h'\leq h$ this gives $\frac{\hmax}{h'} \geq \frac{\hmax}{h}$ and therefore 
	$ \mathcal N_{h'}(3\nu\rho^{h'})\geq \left\lfloor\frac{\hmax}{h}\right\rfloor+1$.
	However by assumption of the lemma we have $ \frac{\hmax}{h}\geq  C\rho^{-d(\nu,C,\rho)h} \geq  C\rho^{-d(\nu,C,\rho)h'}$. 
	It follows that $ \mathcal N_{h'}(3\nu\rho^{h'})> \left\lfloor C\rho^{-d(\nu,C,\rho)h'}\right\rfloor$. 
	This contradicts  $f$ being of near-optimality dimension $d(\nu,C,\rho)$  with associated constant $C$ as defined in Definition~\ref{def:neardim}.
	Indeed the condition $\mathcal N_{h'}(3\nu\rho^{h'}) \leq     C\rho^{-dh'}$ in 
	Definition~\ref{def:neardim} is equivalent to the condition
	$\mathcal N_{h'}(3\nu\rho^{h'}) \leq \left\lfloor    C\rho^{-dh'}\right\rfloor$ as  $\mathcal N_{h'}(3\nu\rho^{h'})$ is an integer. 
\end{proof}

\restathsequool*
\setcounter{scratchcounter}{\value{theorem}}\setcounter{theorem}{\the\numexpr\getrefnumber{c:readable}-1}
\restacorohstar*
\setcounter{theorem}{\the\numexpr\value{scratchcounter}}
\begin{proof}
	Let $x^\star$ be a global optimum with associated $(\nu,\rho)$. For simplicity, let $d=d(\nu,C,\rho)$.
	We have
	\[
	f(x(\timeHorizon)) 
	\stackrel{\textbf{(a)}}{\geq}
	f_{\depthOp_{\hmax}+1,i^\star} 
	\stackrel{\textbf{(b)}}{\geq}
	f(x^\star)- \nu\rho^{\depthOp_{\hmax}+1}.
	\]
	where \textbf{(a)} is because $x(\depthOp_{\hmax}+1,i^\star)\in\tree$ and $ x(\timeHorizon)=\argmax_{\partition_{h,i}\in \tree}f_{h,i}$, and 
	  \textbf{(b)} is by Assumption~\ref{as:smooth}.
	Note that the tree has depth $\hmax+1$ in the end.
	From the previous inequality we have 
	$r_\timeHorizon = \sup_{x\in\mathcal X} f\left(x\right)- f\left(x(n)\right)\leq \nu\rho^{\depthOp_{\hmax}+1}$.
	For the rest of the proof, we want to lower bound $\depthOp_{\hmax}$. Lemma~\ref{lem:hstar} provides a sufficient condition  on  $h$ to get lower bounds. This condition is an inequality in which as $h$ gets larger (more depth) the condition is more and more likely not to hold. For our bound on the regret of \SequOOL to be small, we want a quantity  $h$ so that the inequality holds but having $h$ as large as possible. So it makes sense to see when the inequality flip signs which is when it turns to equality.
	This is what we solve next. We solve Equation \ref{eq:eq} and then verify that it gives a valid indication of the behavior of our algorithm in term of its optimal $h$. 
	We denote $\bar h$ the positive real number satisfying 
	\begin{equation}\label{eq:eq}
	\frac{\hmax}{\bar h}=C\rho^{-d\bar h}.
	\end{equation} 
	First we will verify that $\left\lfloor\bar h\right\rfloor$ is a reachable depth by \SequOOL in the sense that $\bar h\leq \hmax$. As $\rho<1$, $d\geq 0$ and $\bar h\geq 0$ we have $\rho^{-d\bar h}\geq 1$.  This gives  $C\rho^{-d\bar h}\geq 1$. Finally as 
	$ \frac{\hmax}{\bar h}=C\rho^{-d\bar h}$, we have $\bar h\leq \hmax$.

	If $d=0$ we have $\bar h=\hmax/C$. 
	If $d>0$  we have  
	$\bar h = \frac{1}{d\log(1/\rho)}\lambertW\left(\hmax d\log(1/\rho)/C\right)$
	where  $\lambertW$ is the standard Lambert $\lambertW$ function. 
	Using standard properties of the $\lfloor\cdot\rfloor$ function, we have 
	\begin{equation}\label{eq:barbar}
	\frac{\hmax}{ \left\lfloor\bar h\right\rfloor} 
	\geq  \frac{\hmax}{\bar h} =C\rho^{-d\bar h}\geq C\rho^{-d \left\lfloor\bar h\right\rfloor}.
	\end{equation}
	We always have $\depthOp_{\hmax}\geq 0$. If $\bar h \geq 1$,  as discussed above  $\left\lfloor\bar h\right\rfloor  \in [\hmax]$,
	therefore $\depthOp_{\hmax}\geq \depthOp_{\left\lfloor\bar h\right\rfloor},
	$ as $\depthOp_{\cdot}$ is increasing.
	Moreover $ \depthOp_{\bar h} =  \bar h$ because of Lemma~\ref{lem:hstar} which assumptions are verified because of Equation~\ref{eq:barbar}  and  $\left\lfloor\bar h\right\rfloor \in [0:\hmax]$. So in general we have $\depthOp_{\hmax}   \geq  \left\lfloor\bar h\right\rfloor$.
	\noindent
	If $d=0$ we have, 
	$r_\timeHorizon
	\leq
	\nu\rho^{\depthOp_{\hmax}+1}
	\leq
	\nu\rho^{\left\lfloor\bar h\right\rfloor+1}
	=
	\nu\rho^{\left\lfloor\frac{\hmax}{C}\right\rfloor+1}
	\leq
	\nu\rho^{\frac{\hmax}{C}}
	=
	\nu\rho^{\frac{1}{C}\left\lfloor\frac{\timeHorizon}{\bar\log\,\timeHorizon}\right\rfloor }.$ 

	\noindent If $d>0$ 
	$
	r_\timeHorizon
	\leq
	\nu\rho^{\depthOp_{\hmax}+1}
	~\leq~ \nu\rho^{ 
		\frac{1}{d\log(1/\rho)}\lambertW\left(\frac{\hmax d\log(1/\rho)}{C}\right)}.
	$
	To obtain the result in Corollary~\ref{c:readable}, we use that
	$W(x)$  verifies for $x\geq e$,
	$\lambertW(x)\geq \log\left(\frac{x}{\log x}\right)$~\citep{Hoorfar08IO}.
	Therefore, if $\hmax d\log(1/\rho)/C>e$ we have, denoting $d_{\rho}=d\log(1/\rho)$,
	\begin{align*}
	\frac{r_\timeHorizon}{\nu}
	&\leq 
	\rho^{ \frac{1}{d_{\rho}}
		\left(
		\log\left(\frac{\hmax d_{\rho}/C}{
			\log\left(\hmax d_{\rho}/C   \right)
		}\right)
		\right)
	}
	=  e^{ \frac{1}{d\log(1/\rho)}
		\left(
		\log\left(\frac{\hmax d_{\rho}/C}{
			\log\left(\frac{\hmax d_{\rho}}{C}    \right)
		}\right)
		\right)
		\log(\rho)}
	= \left(\frac{\hmax d_{\rho}/C}{\log\left(\frac{\hmax d_{\rho}}{C}\right)}\right)^{- \frac{1}{d}}.
	\end{align*}\end{proof}
\section{\StroquOOL is not using a budget larger than $n$ } 
\label{app:budstro}
Summing over the depths except the depth $0$,  \StroquOOL never uses more evaluations than the budget $\hmax\bar\log^2(\hmax)$ during this depth exploration as 
\begin{align*}
\sum_{h=1}^{\hmax}\sum_{p=0}^{\left\lfloor\hmax/h\right\rfloor}
\left\lfloor \frac{\hmax}{hp}\right\rfloor 
&\leq 
\sum_{h=1}^{\hmax}\sum_{p=0}^{\left\lfloor\hmax/h\right\rfloor}
 \frac{\hmax}{hp}
=
 \sum_{h=1}^{\hmax}\frac{\hmax}{h}\sum_{p=0}^{\left\lfloor\hmax/h\right\rfloor}
 \frac{1}{p}
 =
 \sum_{h=1}^{\hmax}\frac{\hmax}{h}
 \bar\log(\left\lfloor\hmax/h\right\rfloor)\\
&\leq
 \bar\log(\hmax) \sum_{h=1}^{\hmax}\frac{\hmax}{h}
=
\hmax\bar\log^2(\hmax).
\end{align*}
We need to add the additional evaluations for the cross-validation at the end, 
\begin{align*}
\sum_{p=0}^{\pmax}
\frac{1}{2}\left\lfloor \frac{\timeHorizon}{2(\bar\log\timeHorizon+1)^2}\right\rfloor
\leq \frac{\timeHorizon}{4}\cdot
\end{align*}
Therefore, in total the budget is not more than $\frac{\timeHorizon}{2}+\frac{\timeHorizon}{4}+\hmax=\timeHorizon$.

%
\section{Lower bound on the probability of event $\xi_\delta$
}
\label{app:proofevent}
In this section, we define and consider event $\xi_\delta$ and prove 
it holds with high probability.
\begin{lemma}\label{l:event}
	Let  $\mathcal{C}$ be the set of cells evaluated by \StroquOOL during one of its runs. $\mathcal{C}$ is a random quantity.
	Let $\xi_\delta$ be the event under which all average
	estimates in the cells receiving at least one evaluation from \StroquOOL are within their classical confidence interval, then  $P(\xi_\delta)\geq 1 - \delta$, where 
	\[
	\xi_\delta
	\triangleq
	\left\{ \forall \partition_{h,i} \in \mathcal{C}, \,p\in[0:\pmax]: \text{if~ } \pullsNumber_{h,i}= 2^p, \text{ then }
	\left|\hat f_{h,i} - f_{h,i}    \right| \leq b\sqrt{    \frac{\log(2\timeHorizon^2/\delta)}{2^{p+1}}}
	\right\}\!\cdot
	\]

\end{lemma}    

\begin{proof}
	The proof of this lemma follows  
	 the proof of the equivalent statement given for
	\StoSOO \citep{Valko88SS}. The crucial point is that 
	while we have potentially exponentially many combinations 
	of cells that can be evaluated, given any particular execution 
	we need to consider only a polynomial number of estimators
	for which we can use Chernoff-Hoeffding concentration inequality.

Let $m$ denote the (random) number of different nodes sampled by the algorithm
up to time $n$.
Let $\tau_j^1$ be the first time when the $j$-th new node $\partition_{H_j,I_j}$ is
sampled, i.e., at time $\tau_j^1-1$ there are only $j-1$ different nodes that
have been sampled whereas at time $\tau_j^1$, the $j$-th new node
$\partition_{H_j,I_j}$  is sampled for the first time. Let $\tau_j^s$, for $1\leq
s\leq T_{H_j,I_j}(n)$, be the time when the node $\partition_{H_j,I_j}$ is sampled
for the $s$-th time. Moreover, we denote $Y_j^s = y_{\tau_j^s} -
f(x_{H_j,I_j})$. Using this notation, we rewrite $\xi$ as:
\begin{align}\label{eq:xistosoo}
\xi_\delta = \Bigg\{ &\forall j,p \mbox{ s.t. }, 1\leq i \leq m,
\,p\in[0:\pmax], \text{if }T_{H_i, J_i}(n)=2^p,
 \bigg| \frac{1}{2^p}\sum_{s=1}^{2^p} Y^s_j \bigg| \leq \sqrt{\frac{\log(2\timeHorizon^2 /\delta)}{2^{p+1}}} \Bigg\}.
\end{align}
Now, for any $j$ and $p$, the $(Y_j^s)_{1\leq s\leq u}$ are i.i.d.~from
some distribution $\partition_{H_j,I_j}$. The node $\partition_{H_j,I_j}$ is random and
depends on the past samples (before time $\tau_j^1$) but the  $(Y_j^s)_s$ are
conditionally independent given this node and consequently:
\begin{align*}
\P&\Bigg( \bigg| \frac{1}{2^p}\sum_{s=1}^{2^p} Y_j^s \bigg| \leq \sqrt{\frac{\log(2\timeHorizon^2
		/\delta)}{2^{p+1}}}\Bigg) = \\
& =   \E_{\partition_{H_j,I_j}}\, \P\bigg( \bigg|
\frac{1}{2^p}\sum_{s=1}^u Y^s_i \bigg| \leq \sqrt{\frac{\log(2\timeHorizon^2 /\delta)}{2^{p+1} }}
\ \Bigg| \partition_{H_j,I_j} \bigg) \\
& \geq 1-\frac{\delta}{2n},
\end{align*}
using Chernoff-Hoeffding's inequality. We finish the proof
by taking a union bound over all values of $1\leq j\leq n$ and $1\leq p\leq \pmax$.
\end{proof}%
\section{Proof of Lemma~\ref{lem:hstarSto}}
\label{app:prooflemmastro}
\setcounter{scratchcounter}{\value{theorem}}\setcounter{theorem}{\the\numexpr\getrefnumber{lem:hstarSto}-1}
\restahstarSto*
\setcounter{theorem}{\the\numexpr\value{scratchcounter}}
\begin{proof}
	We place ourselves on  event $\xi_\delta$ defined in Lemma~\ref{l:event} and for which we  proved  that $P(\xi_\delta)\geq 1 - \delta$.  We fix $p$. 
	We prove the statement of the lemma, given that event $\xi_\delta$ holds, by induction in the following sense. For a given $h$ and $p$, we assume the hypotheses of the lemma for that $h$ and $p$ are true and we prove by induction that $\depthOp_{h',p}= h'$ for $h'\in[h]$. \\[.1cm] 
	$1^\circ$ For $h'= 0$, we trivially have that $\depthOp_{h',p} \geq 0$.\\     
	$2^\circ$  Now consider $h'>0$, and assume $\depthOp_{h'-1,p}= h'-1$ with the objective to prove that $\depthOp_{h',p}= h'$.
	%
	%
	%
	Therefore, at the end of the processing of depth $h'-1$, during which we were opening the cells of depth $h'-1$ we managed to open the cell $\partition_{h'-1, i_{h'-1}^\star}$ with at least $2^p$ evaluations. $\partition_{h'-1, i_{h'-1}^\star}$ is
	the optimal node of depth $h'-1$ (i.e., such that
	$x^\star \in \mathcal P_{h'-1,i^\star})$.
	Let $m$ be the largest integer such that $2^p\leq \frac{\hmax}{2h'm}$.
	We have $\frac{\hmax}{2h'm}	\leq \left\lfloor\frac{\hmax}{h'm}\right\rfloor$ and also $2^p\geq \frac{\hmax}{2h'(m+1)}\geq \frac{\hmax}{4h'm}$.
	During phase $h'$, the $m$ cells from $\left\{\partition_{h',i}\right\}$ with highest values $\left\{\hat f(x_{h',i})\right\}_{h',i}$ and having been evaluated at least $\left\lfloor\frac{\hmax}{h'm}\right\rfloor\geq 2^{p}$ are opened  at least $\left\lfloor\frac{\hmax}{h'm}\right\rfloor\geq 2^{p}$ times.
	For the purpose of contradiction, let us assume that $\partition_{h',i_{h'}^\star}$ is not one of them. This would mean that there exist at least 
	$m$ cells from $\left\{\partition_{h',i}\right\}$, distinct from $\partition_{h',i_h^\star}$, satisfying 
	$\hat f_{h',i}
	\geq
	\hat f_{h',i^\star_{h'}}$ and each having been evaluated at least $2^p$  times. 
	This means that, for these cells we have
	$
		f_{h',i} + \nu\rho^{h'}
\geq 
	f_{h',i} + \nu\rho^{h}
	\stackrel{\textbf{(a)}}{\geq} 
	f_{h',i} + b\sqrt{    \frac{\log(2\timeHorizon^2/\delta)}{    2^{p+1}}}
	\stackrel{\textbf{(b)}}{\geq} 
	\hat f_{h',i}
	\geq
	\hat f_{h',i^\star_{h'}}
	\stackrel{\textbf{(b)}}{\geq} 
	f_{h',i^\star_{h'}}- b\sqrt{    \frac{\log(2\timeHorizon^2/\delta)}{    2^{p+1}}}
		\stackrel{\textbf{(a)}}{\geq} 
	f_{h',i^\star_{h'}}- \nu\rho^{h}
	\geq 
	f_{h',i^\star_{h'}}- \nu\rho^{h'}$,
	where \textbf{(a)} is by assumption of the lemma, 
	\textbf{(b)} is because $\xi$ holds.
	As $f_{h',i^\star_{h'}} \geq f(x^\star) -\nu\rho^{h'}$ by Assumption~\ref{as:smooth}, this means we have $ \mathcal N_{h'}(3\nu\rho^{h'})\geq m+1\geq \frac{\hmax}{4h'2^p}+1$ (the $+1$ is for $\partition_{h',i_{h'}^\star}$). 
	 As $h'\leq h$ this gives $\frac{\hmax}{h'2^p} \geq \frac{\hmax}{h2^p}$ and therefore $ \mathcal N_{h'}(3\nu\rho^{h'})\geq \left\lfloor\frac{\hmax}{4h2^p}\right\rfloor+1$.
	However by assumption of the lemma we have $ \frac{\hmax}{4h2^p}\geq  C\rho^{-d(\nu,C,\rho)h} \geq  C\rho^{-d(\nu,C,\rho)h'}$. 
	It follows that $ \mathcal N_{h'}(3\nu\rho^{h'})> \left\lfloor C\rho^{-d(\nu,C,\rho)h'}\right\rfloor$. 
	This leads to having a contradiction with the function $f$ being of near-optimality dimension $d(\nu,C,\rho)$
	as defined in~Definition~\ref{def:neardim}.
	Indeed, the condition $\mathcal N_{h'}(3\nu\rho^{h'}) \leq     C\rho^{-dh'}$ in 
	Definition~\ref{def:neardim} is equivalent to the condition
	$\mathcal N_{h'}(3\nu\rho^{h'}) \leq \left\lfloor    C\rho^{-dh'}\right\rfloor$ as  $\mathcal N_{h'}(3\nu\rho^{h'})$ is an integer. 
	Reaching the contradiction proves the claim of the lemma.
\end{proof}
\section{Proof of Theorem~\ref{th:highnoise} and Theorem~\ref{th:lownoise}}%
\label{app:proofALLmastro}
\restathhighnoise*
\restathlownoise*
%
\begin{proof}[Proof of~Theorem~\ref{th:highnoise} and Theorem~\ref{th:lownoise}]
	We first place ourselves on the event $\xi$ defined in Lemma~\ref{l:event} and where it is proven that $P(\xi)\geq 1 - \delta$. We bound the simple regret of \StroquOOL on $\xi$. 
	We consider a global optimum $x^\star$ with associated $(\nu,\rho)$. For simplicity we write $d=d(\nu,C,\rho)$.
	We have for all 
	$p\in[0:\pmax]$
	\begin{align*}
	f(x(\timeHorizon))&+ b\sqrt{\frac{\log(2\timeHorizon^2/\delta)}{\hmax}}       \stackrel{\textbf{(a)}}{\geq}
	\hat f(x(\timeHorizon))
	\stackrel{\textbf{(c)}}{\geq}
	\hat f(x(\timeHorizon,p))
	\stackrel{\textbf{(b)}}{\geq}
	\hat f\left(x(\depthOp_{\hmax,p}+1,i^\star)\right)\\
	&\stackrel{\textbf{(a)}}{\geq}
	f(x(\depthOp_{\hmax,p}+1,i^\star)) -b\sqrt{\frac{\log(2\timeHorizon^2/\delta)}{\hmax}}  
	\stackrel{\textbf{(d)}}{\geq}
	f\left(x^\star\right)- \nu\rho^{\depthOp_{\hmax,p}+1}  -b\sqrt{\frac{\log(2\timeHorizon^2/\delta)}{\hmax}} 
	\end{align*}
	where \textbf{(a)} is because the $x(\timeHorizon,p)$ are evaluated $\hmax$ times at the end of \StroquOOL and because $\xi$ holds,  \textbf{(b)} is because $x_{\depthOp_{\hmax,p}+1,i^\star}\in \{\left(h,i\right)\in \tree,\pullsNumber_{h,i}\geq 2^p\}$ and  $x(n,p) = \argmax\limits_{\partition_{h,i}\in \tree,\pullsNumber_{h,i}\geq 2^p} \hat f_{h ,i}$, \textbf{(c)} is because $x(\timeHorizon) = \argmax\limits_{\{x(\timeHorizon,p) ,p\in[0:\pmax]\}} \hat f(x(\timeHorizon,p))$, and 
	\textbf{(d)} is by Assumption~\ref{as:smooth}.
	
	From the previous inequality we have 
	$r_\timeHorizon =  f\left(x^\star\right)- f\left(x(\timeHorizon)\right)\leq \nu\rho^{\depthOp_{\hmax,p}+1}+2b\sqrt{\frac{\log(2\timeHorizon^2/\delta)}{\hmax}}$, for     $p\in[0:\pmax]$.
	%

	For the rest of proof we want to lower bound $\max_{p\in[0:\pmax]}\depthOp_{\hmax,p}$. Lemma~\ref{lem:hstarSto} provides some sufficient conditions  on $p$ and $h$ to get lower bounds. These conditions are inequalities in which as $p$ gets smaller (fewer samples) or $h$ gets larger (more depth) these conditions are more and more likely not to hold. For our bound on the regret of \StroquOOL to be small, we want quantities $p$ and $h$ where the inequalities hold but using as few samples as possible (small $p$) and having $h$ as large as possible. Therefore we are interested in determining when the inequalities flip signs which is when they turn to equalities. 
	This is what we solve next. 
	We denote $\tilde h$ and $\tilde p$ the real numbers satisfying
	\begin{equation}\label{eq:eqStro}
	\frac{\hmax\nu^2\rho^{2\tilde h}}{4\tilde hb^2\log(2\timeHorizon^2/\delta)}=C\rho^{-d\tilde h}
	\quad \text{and} \quad b\sqrt{\frac{\log(2\timeHorizon^2/\delta)}{2^{\tilde p}}} = \nu\rho^{\tilde h}.  
	\end{equation}
	Our approach is to solve Equation~\ref{eq:eqStro} and then verify that it gives a valid indication of the behavior of our algorithm in term of its optimal $p$ and $h$.     We have  
	\[\tilde h = \frac{1}{(d+2)\log(1/\rho)}\lambertW\left(\frac{\nu^2\hmax (d+2)\log(1/\rho)}{4Cb^2\log(2\timeHorizon^2/\delta)}\right)\]
	where standard $\lambertW$ is the Lambert $\lambertW$ function.  
	
	However after a close look at the Equation~\ref{eq:eqStro}, we notice that it is possible to get values $\tilde p< 0$ which would lead to a number of evaluations $2^p<1$. This actually corresponds to an interesting case when the noise has a small range and where we can expect to obtain an improved result, that is: obtain a regret rate close to the deterministic case. This low range of noise case then has to be considered separately.
	
	Therefore, we distinguish two cases which corresponds to different noise regimes depending on the value of $b$. Looking at the equation on the right of~\eqref{eq:eqStro}, we have that $\tilde p< 0$ if $\frac{\nu^2\rho^{2\tilde h}}{b^2\log(2\timeHorizon^2/\delta)} > 1$. Based on this condition we now  consider the two cases. However for both of them we define some generic $\ddot h$ and     $\ddot p$.
	\paragraph{High-noise regime $\frac{\nu^2\rho^{2\tilde h}}{b^2\log(2\timeHorizon^2/\delta)}\leq 1$:}
	In this case, we denote $\ddot h = \tilde h$ and $\ddot p = \tilde p$.
	%
	As $\frac{1}{2^{\tilde p} } 
	=
	\frac{\nu^2\rho^{2\tilde h}}{b^2\log(2\timeHorizon^2/\delta)}
	\leq 1$
	by construction,     we have $\tilde p\geq 0$.
	Using standard properties of the $\lfloor\cdot\rfloor$ function, 
	we have 
	\begin{equation}\label{eq:barbarSto2}
	b\sqrt{\frac{\log(2\timeHorizon^2/\delta)}{2^{\left\lfloor\tilde p\right\rfloor+1}}}
	\leq
	b\sqrt{\frac{\log(2\timeHorizon^2/\delta)}{2^{\tilde p}}} =
	\nu\rho^{\tilde h} \leq \nu\rho^{\left\lfloor\tilde h\right\rfloor}
	\end{equation}
	\begin{equation}\label{eq:barbarSto}
	\text{    and,~~ }\frac{\hmax}{4\left\lfloor\tilde h\right\rfloor 2^{\left\lfloor\tilde p\right\rfloor}}  
	\geq
	\frac{\hmax}{4\left\lfloor\tilde h\right\rfloor 2^{\tilde p}}  
	=
	\frac{\hmax\nu^2\rho^{2\tilde h}}{4\left\lfloor\tilde h\right\rfloor b^2\log(2\timeHorizon^2/\delta)}  
	\geq  \frac{\hmax\nu^2\rho^{2\tilde h}}{4\tilde h b^2 \log(2\timeHorizon^2/\delta)} 
	= C\rho^{-d \tilde h}
	\geq C\rho^{-d \left\lfloor\tilde h\right\rfloor}.
	\end{equation}
	\paragraph{Low-noise regime $\frac{\nu^2\rho^{2\tilde h}}{b^2\log(2\timeHorizon^2/\delta)}> 1$ or $b=0$:}
	In this case, we can reuse arguments close to the argument used in the deterministic feedback case in the proof of \SequOOL~(Theorem~\ref{th:sequool}), we denote  $\ddot h = \bar h$ and $\ddot p = \bar p$ where $\bar h$ and $\bar p$ verify,    
	\begin{equation}\label{eq:barbarSto56}
	\frac{\hmax}{4\bar h}=C\rho^{-d\bar h}
	\quad \text{and} \quad \bar p = 0.  
	\end{equation}
	If $d=0$ we have $\bar h=\hmax/C$. 
	If $d>0$  we have  
	$\bar h = \frac{1}{d\log(1/\rho)}\lambertW\left(\frac{\hmax d\log(1/\rho)}{4C}\right)$
	where standard  $\lambertW$ is the standard Lambert $\lambertW$ function.    Using standard properties of the $\lfloor\cdot\rfloor$ function, 
	we have 
	\begin{equation}\label{eq:barbarSto57}
	b\sqrt{\frac{\log(2\timeHorizon^2/\delta)}{2^{\left\lfloor\ddot p\right\rfloor+1}}}
	\leq
	b\sqrt{\log(2\timeHorizon^2/\delta)} <
	\nu\rho^{\tilde h} \stackrel{\textbf{(a)}}{\leq} \nu\rho^{\bar h} \leq \nu\rho^{\left\lfloor\bar h\right\rfloor}
	\end{equation}
	where \textbf{(a)} is because of the following reasoning. First note that one can assume $b>0$ as for the case $b=0$, the Equation~\ref{eq:barbarSto57} is trivial.    As we have $\frac{\hmax\nu^2\rho^{2\tilde h}}{4\tilde hb^2\log(2\timeHorizon^2/\delta)}=C\rho^{-d\tilde h}$ and $\frac{\nu^2\rho^{2\tilde h}}{b^2\log(2\timeHorizon^2/\delta)}> 1$, then, 
	$\frac{\hmax}{4\tilde h}< C\rho^{-d\tilde h}$. From the inequality $\frac{\hmax}{4\tilde h}< C\rho^{-d\tilde h}$ and the fact that $\bar h$ corresponds to the case of equality  $\frac{\hmax}{4\bar h}= C\rho^{-d\bar h}$, we deduce that $\bar h\leq \tilde h$, since the left term of the inequality decreases with $h$ while the right term increases. Having $\bar h\leq \tilde h$ gives $\rho^{\bar h}\geq \rho^{\tilde h}$.
	
	Given these particular definitions of $\ddot h$ and $\ddot p$ in two distinct cases we now bound the regret.

	
	%
	
		First we will verify that $\left\lfloor\ddot h\right\rfloor$ is a reachable depth by \StroquOOL in the sense that $\ddot h\leq \hmax$ and $\ddot p\leq \log_2(\hmax/h)$ for all $h\leq\ddot h$. 
	As $\rho<1$, $d\geq 0$ and $\ddot h\geq 0$ we have $\rho^{-d\ddot h}\geq 1$. This gives  $C\rho^{-d\ddot h}\geq 1$. Finally as 
	$ \frac{\hmax}{\ddot h  2^{\ddot p}}=C\rho^{-d\ddot h}$, we have $\ddot h\leq \hmax/ 2^{\ddot p}$.
	Note also that from the previous equation we have that if $\ddot h\geq 1$, $\ddot p\leq \log_2(\hmax/h)$ for all $h\leq\ddot h$ . Finally in both regimes we already proved that $\ddot p\geq 0$.

	We always have $\depthOp_{\hmax,\left\lfloor\ddot p\right\rfloor}\geq 0$. If $\ddot h \geq 1$, as discussed above  $\left\lfloor\ddot h \right\rfloor \in \left[\hmax\right]$,
	therefore $\depthOp_{\hmax,\left\lfloor\ddot p\right\rfloor}\geq \depthOp_{\left\lfloor\ddot h\right\rfloor,\left\lfloor\ddot p\right\rfloor},
	$ as $\depthOp_{\cdot,\left\lfloor p\right\rfloor}$ is increasing for all $p\in[0,\pmax ]$.
	Moreover on event~$\xi$, $ \depthOp_{\left\lfloor\ddot h\right\rfloor,\left\lfloor\ddot p\right\rfloor} =  \left\lfloor\ddot h\right\rfloor$ because of Lemma~\ref{lem:hstarSto}  which assumptions on $\left\lfloor\ddot h\right\rfloor$ and $\left\lfloor\ddot p\right\rfloor$ are verified because of Equations~\ref{eq:barbarSto2} and~\ref{eq:barbarSto} in the high-noise regime and because of Equations~\ref{eq:barbarSto56} and~\ref{eq:barbarSto57} in the low-noise regime,  and, in general,  $\left\lfloor\ddot h \right\rfloor \in \left[\left\lfloor\hmax/2^{\ddot p}\right\rfloor\right]$ and  $\left\lfloor\ddot p\right\rfloor\in [0: \pmax]$. So in general we have $\depthOp_{\left\lfloor\hmax/2^{\ddot p}\right\rfloor,\left\lfloor\ddot p\right\rfloor}  \geq  \left\lfloor\ddot h \right\rfloor$.

			We can now bound the regret in the two regimes.
	\paragraph{High-noise regime} 
	In general, we have, on event $\xi$,
	\begin{align*}
	r_\timeHorizon
	&~\leq~ \nu\rho^{ 
		\frac{1}{(d+2)\log(1/\rho)}\lambertW\left(\frac{\nu^2\hmax (d+2)\log(1/\rho)}{C\log(2\timeHorizon^2/\delta)}\right)}+2b\sqrt{\frac{\log(2\timeHorizon^2/\delta)}{\hmax}}\cdot
	\end{align*}
	While in the deterministic feedback case, the regret was scaling with $d$ when $d\geq 0$, in the stochastic feedback case, the regret scale with $d+2$. This is because the uncertainty due to the presence of noise diminishes as $n^{-\frac{1}{2}}$ when we collect $n$ observations.
	
	Moreover, as proved by~\citet{Hoorfar08IO}, the Lambert $W(x)$ function verifies for $x\geq e$,
	$\lambertW(x)\geq \log\left(\frac{x}{\log x}\right)$.
	Therefore, if $\frac{\nu^2\hmax (d+2)\log(1/\rho)}{4C\log(2\timeHorizon^2/\delta)}>e$ we have, denoting $d'=(d+2)\log(1/\rho)$, 
	%
	\begin{align*}
	r_\timeHorizon-2b\sqrt{\frac{\log(2\timeHorizon^2/\delta)}{\hmax}}
	&\leq 
	\nu \rho^{ \frac{1}{d'}
		\left(
		\log\left(\frac{\frac{\hmax d'\nu^2}{4C\log(2\timeHorizon^2/\delta)}}{
			\log\left(\frac{\hmax d'\nu^2}{4C\log(2\timeHorizon^2/\delta)}    \right)
		}\right)
		\right)
	} \\
	&= \nu e^{ \frac{1}{(d+2)\log(1/\rho)}
		\left(
		\log\left(\frac{\frac{\hmax d'\nu^2}{4C\log(2\timeHorizon^2/\delta)}}{
			\log\left(\frac{\hmax d'\nu^2}{4C\log(2\timeHorizon^2/\delta)}    \right)
		}\right)
		\right)
		\log(\rho)}
	= \nu \left(\frac{\frac{\hmax d'\nu^2}{4C\log(2\timeHorizon^2/\delta)}}{\log\left(\frac{\hmax d'\nu^2}{4C\log(2\timeHorizon^2/\delta)}\right)}\right)^{- \frac{1}{d+2}}.
	\end{align*}

	\paragraph{Low-noise regime}
	We have $2b\sqrt{\frac{\log(2\timeHorizon^2/\delta)}{\hmax}}\leq 
	2\frac{\nu\rho^{\tilde h}}{\sqrt{\log(2\timeHorizon^2/\delta)}} \sqrt{\frac{\log(2\timeHorizon^2/\delta)}{\hmax}}\leq
	2\nu\rho^{\tilde h}\leq
	2\nu\rho^{\bar h}$.
	Therefore
	$r_\timeHorizon
	\leq
	\nu\rho^{\depthOp_{\hmax,\bar p}+1} +2b\sqrt{\frac{\log(2\timeHorizon^2/\delta)}{\hmax}}
	\leq
	3\nu\rho^{\bar h} $.
	Discriminating between $d=0$ and $d>0$ leads to the claimed results.

	\paragraph{Results in Expectation}	
		We want to  obtain additionally, our final result as an upper bound on the expected simple regret $\Exp r_n$.  Compared to the results in high probability, the following  extra assumption that the function $f$ is bounded is made: For all $x\in \dom, |f(x)|\leq f_{\max}$. Then $\delta$ is set as $\delta=\frac{4b}{f_{\max}\sqrt{\timeHorizon}}$.
			We bound the expected regret now discriminating on whether or not the event $\xi$ holds. We have 
		\begin{align*}
		\Exp r_\timeHorizon &\leq (1-\delta) \left(\nu\rho^{\depthOp_{\hmax,\ddot p}+1}+2b\sqrt{\frac{\log(f_{\max}\timeHorizon^{5/2}/b)}{\hmax}}\right) +\delta\times f_{\max}\\
&		\leq
		\nu\rho^{\depthOp_{\hmax,\ddot p}+1}
		+2b\sqrt{\frac{\log(f_{\max}\timeHorizon^{5/2}/b)}{\hmax}}+\frac{4b}{\sqrt{\timeHorizon}}\\
		&\leq
		\nu\rho^{\depthOp_{\hmax,\ddot p}+1}+6b\sqrt{\frac{\log(f_{\max}\timeHorizon^{5/2}/b)}{\hmax}}\cdot
		\end{align*}
\end{proof}
%

%


\end{document}